\documentclass[10pt,twocolumn,letterpaper]{article}

\usepackage{cvpr}
\usepackage{times}
\usepackage{epsfig}
\usepackage{graphicx}
\usepackage{amsmath}
\usepackage{amsthm}
\theoremstyle{plain}
\newtheorem{theorem}{Theorem}
\newtheorem{assumption}{Assumption}

\newtheorem{proposition}{Proposition}
\usepackage{amssymb}
\usepackage{times}
\usepackage{float}
\usepackage{multirow}
\usepackage{subfigure}
\usepackage{mathrsfs}
\usepackage{cite}
\usepackage{mathrsfs}
\usepackage{makecell}
\usepackage{float}
\usepackage[english]{babel}
\usepackage[ruled]{algorithm2e}
\usepackage{bm}
\usepackage{nicefrac}

\usepackage[ruled]{algorithm2e}
\usepackage{array,booktabs,longtable,tabularx,tabulary}
\usepackage{ltablex}
\usepackage{siunitx}
\usepackage[flushleft]{threeparttablex}
\usepackage{stfloats}
\usepackage{lipsum}

\usepackage[pagebackref=true,breaklinks=true,letterpaper=true,colorlinks,bookmarks=false]{hyperref}

 \cvprfinalcopy 

\def\cvprPaperID{1828} 

\ifcvprfinal\fi
\begin{document}

\title{DistillHash: Unsupervised Deep Hashing by Distilling Data Pairs}

\author{
Erkun Yang$^{1,2}$,
Tongliang Liu$^2$,
Cheng Deng$^1$\thanks{Corresponding author}\ ,
Wei Liu$^{3*}$,
Dacheng Tao$^2$
\\
$^1$ School of Electronic Engineering, Xidian University, Xi¡¯an 710071, China \\
$^2$ UBTECH Sydney AI Centre, School of Computer Science, FEIT, University of Sydney, \\Darlington, NSW 2008, Australia,
$^3$ Tencent AI Lab, Shenzhen, China\\
{\tt\small ekyang@stu.xidian.edu.cn,
tongliang.liu@sydney.edu.au,
chdeng.xd@gmail.com,}\\
 {\tt\small wl2223@columbia.edu,
dacheng.tao@sydney.edu.au}
}

\maketitle

\begin{abstract}
Due to the high storage and search efficiency, hashing has become prevalent for large-scale similarity search. Particularly, deep hashing methods have greatly improved the search performance under supervised scenarios.
In contrast, unsupervised deep hashing models can
hardly achieve satisfactory performance due to the lack of reliable supervisory similarity signals.
To address this issue, we propose a novel deep unsupervised hashing model, dubbed DistillHash, which can learn a distilled data set consisted of data pairs, which have confidence similarity signals. Specifically, we investigate the relationship between the initial noisy similarity signals learned from local structures and the semantic similarity labels assigned by a Bayes optimal classifier. We show that under a mild assumption, some data pairs, of which labels are consistent with those assigned by the Bayes optimal classifier, can be potentially distilled. Inspired by this fact, we design a simple yet effective strategy to distill data pairs automatically and further adopt a Bayesian learning framework to learn hash functions from the distilled data set. Extensive experimental results on three widely used benchmark datasets show that the proposed DistillHash consistently accomplishes the state-of-the-art search performance.

\end{abstract}
\section{Introduction}
The explosive growth of visual data (e.g., photos and videos) has led to renewed interest
in efficient indexing and searching
algorithms~\cite{wang2018distributed,ng2018incremental,shen2015supervised,li2019coupled,liu2016multimedia,wang2016learning,liu2012compact,zhang2016discrete,li2018sub,liu2016multimedia,liu2018deep,chi2018hashing,wang2017exploring,shen2018zero,chen2018attention,du2018expressive,yang2018adversarial,Yang2018NewL,you2017privileged}, among which, hashing-based
approximate nearest neighbor (ANN) searching, which balances retrieval quality and computational cost,
has attracted increasing attention.


Generally, hashing methods can be divided into supervised and unsupervised models. The supervised hashing models~\cite{liu2012supervised,li2015feature,Chen_2018_CVPR,yang2018shared}, which aim to learn
hash functions with semantic labels, have shown remarkable performance. However, existing supervised hashing methods, especially
deep hashing rely on massive labeled data examples to train their models. Thus, when there exist no enough training examples,
their performance may be dramatically degraded caused by over-fitting to those training examples.

To address this challenge, unsupervised hashing methods usually adopt learning frameworks without requiring any supervised information. Traditional
unsupervised hashing methods~\cite{andoni2006near,gong2013iterative,liu2014discrete,heo2012spherical} with hand-crafted features cannot well preserve the similarities of real-world data samples due to the low model capacity and the
separated representation and binary codes optimization processes. To take advantages of the recent progress of deep learning~\cite{krizhevsky2012imagenet,wang2019evolutionary,you2017learning},
unsupervised deep hashing methods~\cite{krizhevsky2011using,lin2016learning,dai2017stochastic,dizaji2018unsupervised,deng2019unsupervised},
which adopt  neural networks as hash functions, have also been proposed. These deep hashing models are usually
trained by minimizing either
the quantization loss or data reconstruction loss. However, since these objectives fail to exploit the semantic similarities between data points, they
can hardly achieve satisfactory results.

In this paper, we propose a novel unsupervised deep hashing model, dubbed DistillHash, which addresses the absence of supervisory signals by distilling data pairs with confident semantic similarity relationships.
In particular, we first exploit the local structure of training data points to assign an initial similarity label for each data pair.
If we treat the semantic similarity labels as true labels,
these initial similarity labels then contain label- and instance-dependent label noise, because many of
them fail to represent semantic similarities. By assuming that we know the probability of a semantic
similarity label given a pair of the data points, the Bayes optimal classifier will assign the semantic
similarity label to the data pair which has a higher probability (or has a probability greater than 0.5).
Based on these results, we  give strict analysis on the relationship between  the noisy labels and the labels assigned by the Bayes optimal classifier. Inspired by the framework of~\cite{cheng2017learning}, we show that,
under a mild assumption, data pairs with confident semantic labels can be potentially distilled. Furthermore,
we theoretically give the criteria to select distilled data pairs and also provide a simple but effective method to collect distilled data pairs automatically.
Finally, given the distilled data pair set, we design a deep neural network  and adopt a Bayesian learning framework to perform the representation and hash code learning simultaneously.
%

Our main contributions can be summarized as follows:
\begin{itemize}
\item By considering signals learned from deep features as noisy pairwise labels, we successfully apply noisy label learning techniques to our method. This shows that data pairs, of which labels are consistent with those assigned by
the Bayes optimal classifier, can be potentially distilled.
\item We theoretically give the criteria to select distilled data pairs for hash learning and further provide a
simple but effective method to collect distilled data pairs automatically.
\item Experiments on three popular benchmark datasets show that our method can
outperform current state-of-the-art unsupervised hashing methods.
\end{itemize}

The rest of this paper is organized as follows. We review the relevant literature in Section~\ref{relate_work}.
We present our novel DistillHash in Section~\ref{approach}.
Section~\ref{experiments} details the experiments, after which concluding
remarks are presented in Section~\ref{conclusion}.

\section{Related Work}
\label{relate_work}

Recently, the amount of literatures have grown up considerably around the theme of hashing~\cite{liu2016query,deng2018triplet,li2018self,deng2015adaptive,liu2016query,liu2016structure}. According to whether supervised information are involved in the learning phase,
existing  hashing models can be divided into two categories: supervised hashing methods and unsupervised hashing methods.

Supervised hashing methods~\cite{liu2012supervised,gui2018fast,xia2014supervised,li2015feature,cao2018deep,norouzi2011minimal,song2014robust,deng2016discriminative,yang2017pairwise} aim to learn hash functions that can map data points
to Hamming space where the semantic similarity can be preserved.
 Kernel-based supervised hashing (KSH)~\cite{liu2012supervised} uses inner products to approximate the Hamming distance and learns hash functions by perserving semantic similarities in Hamming space.
 Fast supervised discrete hashing (FSDH)~\cite{gui2018fast} uses a simple yet effective regression from the class labels of training data points to the corresponding hash code to accelerate
 the learning process.
 Convolutional neural networks-based hashing (CNNH)~\cite{xia2014supervised} decomposes the hash function learning into two stages. Firstly, a pairwise similarity matrix is
 constructed and decomposed into the product of approximate hash codes. Secondly, CNNH
 simultaneously learns  representations and hash functions by training the model to predict the learned hash
 codes as well as the discrete image class labels.
Deep Cauchy hashing (DCH)~\cite{cao2018deep} adopts Cauchy distribution to continue to optimize data pairs in a relatively small Hamming ball.
Unsupervised hashing methods~\cite{andoni2006near,liu2014discrete,gong2013iterative,heo2012spherical,liu2011hashing} try to  encode original data into
binary codes by training with unlabeled data points.
Iterative quantization (ITQ)~\cite{gong2013iterative} first uses principal component analysis (PCA) to  map the data to a low dimensional space and then exploits an alternating minimization scheme to find a rotation matrix, which maps the data to binary codes with minimum quantization error.
Discrete graph
hashing (DGH)~\cite{liu2014discrete} casts the graph hashing problem into a
discrete optimization framework and explicitly deals with
the discrete constraints, so it can directly output binary codes.
Spherical hashing (SpH)~\cite{heo2012spherical} minimizes
the spherical distance between the original real-valued
features and the learned binary codes.
Anchor graph hashing (AGH)~\cite{liu2011hashing} utilizes anchor graphs
to obtain tractable low-rank adjacency matrices and approximate the data structure.
Though current traditional unsupervised hashing methods have made great progress, they usually depend on pre-defined features and cannot
simultaneously optimize the feature and hash code learning
processes, thus missing an opportunity to learn more effective hash codes.

 Unsupervised deep  hashing methods~\cite{salakhutdinov2009semantic,liong2015deep,krizhevsky2011using,dizaji2018unsupervised,lin2016learning,dai2017stochastic,yang2018semantic}
   adopt deep architectures to extract features and perform hash mapping.
 Semantic hashing~\cite{salakhutdinov2009semantic} uses  pre-trained restricted Boltzmann machines (RBMs)~\cite{nair2010rectified} to construct an auto-encoder network, which is
 then used to generate efficient hash codes and reconstruct the original inputs.
 Deep binary descriptors (DeepBit)~\cite{lin2016learning} considers original images and the corresponding rotated images as similar pairs and tries to learn hash codes to
 preserve this similarity. Stochastic generative hashing (SGH)~\cite{dai2017stochastic} utilizes a generative mechanism to learn hash functions
through the minimum description length
principle. The hash codes are optimized to maximally compress the dataset as well as to regenerate the inputs.
 Semantic structure-based unsupervised deep hashing (SSDH)~\cite{yang2018semantic}  takes advantage of the semantic information
in deep features and learns semantic structures based on the pairwise distances and a Gaussian estimation.
The semantic structure is then used to guide the hash code learning process.
By integrating the feature and hash code learning processes, deep unsupervised hashing methods usually produce better results.

Training classifiers from noisy labels is also a closely related task.
We refer the noisy labels to the setting where the labels of data points are corrupted~\cite{biggio2011support,yu2017transfer,han2018masking,han2018co}. Since in many situations,
it is both expensive and difficult to obtain reliable labels,  a growing body of literature has been devoted to learning with noisy labels.
Those methods can be organized into two major groups: label noise-tolerant classification~\cite{ali1996error,melville2004experiments} and
label noise cleansing methods~\cite{NIPS2013_5073,cheng2017learning,liu2016classification,DBLP:conf/uai/2017}.
The former adopts the strategies like decision trees or boosting-based ensemble techniques, while the latter
tries to filter the label noise by exploiting the prior information from training samples. For a comprehensive understanding, we recommend readers to read~\cite{frenay2014classification}.
By treating the initial similarity relationship as noisy labels, our method can explicitly model the relationship between
noisy labels and labels assigned by the Bayes optimal classifiers, which then enables us to extract data pairs with confident similarity signals.

%
\section{Approach}
\label{approach}

Let $\mathcal{ X }= {\{\bm{x}_{i}\}}^{N}_{i=1}$ denote the training set with $N$ instances,
deep  hashing aims to learn nonlinear hash functions ${h}: \bm{x}\mapsto \bm{b}\in {\{-1, 1\}}^{K} $, which can encode original data
points $\bm{x}$  to
compact $K$-bit hash codes.

Traditional supervised deep hashing methods usually accept data pairs $\{(\bm{x}_{i}, \bm{x}_{j}), S_{ij}\}$ as inputs,
where
$S_{ij}\in\{+1, -1\}$ is a binary label to indicate whether $\bm{x}_{i}$ and $\bm{x}_{j}$ are similar or not. However, due to the laborious labeling process
 and the need of requisite domain knowledge, it's not feasible to directly acquire  labels in many tasks. Thus, in this paper, we study the  hashing problem under unsupervised settings.

 Inspired by the Bayesian classifier theory~\cite{domingos1997optimality}, reliable labels for data pairs can be confidently assigned by  an Bayes optimal classifier, i.e.,
 \begin{equation}
 \label{equ: 1}
 \begin{aligned}
  {S}_{ij} = \begin{cases} 1, \quad&\text{if } \eta(\bm{x}_{i}, \bm{x}_{j})\ge0.5 ,\\ -1,   &\text{if }\eta(\bm{x}_{i}, \bm{x}_{j})<0.5, \end{cases}
 \end{aligned}
\end{equation}
where  $\eta(\bm{x}_{i}, \bm{x}_{j}) = P({S}_{ij} = +1 | \bm{x}_{i}, \bm{x}_{j} )$. This Bayes optimal classifier
implies that if we
have  access to $\eta(\bm{x}_{i}, \bm{x}_{j})$, we can infer the true data labels with Eq.~\ref{equ: 1}.
However, under unsupervised settings, we cannot access $\eta(\bm{x}_{i}, \bm{x}_{j})$.

For unsupervised learning, some recent works~\cite{pinheiro2018unsupervised,li2018instance,yang2018semantic} demonstrate
that local structures learned from original features can help to capture the similarity relationship between points.
Motivated by this, we can roughly label the training
data pairs based on their local structures and construct
a similarity matrix $\bm{\tilde{S}}$ as
\begin{equation}
\label{equ: 2}
\begin{aligned}
{{\tilde{S}}}_{ij} = \begin{cases} 1, \quad&\text{if } d(i,j) \le t_1 ,\\ -1,  \quad&\text{if } d(i,j)>t_2, \end{cases}
\end{aligned}
\end{equation}
where $d(i,j)$ denotes the distance of extracted features for $\bm{x}_{i}$ and $\bm{x}_{j}$, $t_1$ and $t_2$ are the thresholds
for the distance. However, since $\bm{\tilde{S}}$ is only constructed from local structures, they are unreliable and may contain  label noise.

Note that, based on $\bm{\tilde{S}}$ we can
learn an estimation of the conditional probability $\tilde{\eta}(\bm{x}_{i}, \bm{x}_{j}) = P(\tilde{S}_{ij} = +1 | \bm{x}_{i}, \bm{x}_{j} )$.
And, there exists a relationship between  $\tilde{\eta}(\bm{x}_{i}, \bm{x}_{j})$ and $\eta(\bm{x}_{i}, \bm{x}_{j})$ as
\begin{equation}
\label{equ: 3}
\begin{aligned}
&\tilde{\eta}(\bm{x}_{i}, \bm{x}_{j}) = P(\tilde{S}_{ij} = +1 | \bm{x}_{i}, \bm{x}_{j} )\\
&= P(\tilde{S}_{ij} = +1|\bm{x}_{i}, \bm{x}_{j}, S_{ij} = +1)P(S_{ij} = +1|\bm{x}_{i}, \bm{x}_{j})\\
&\quad + P(\tilde{S}_{ij} = +1|\bm{x}_{i}, \bm{x}_{j}, S_{ij} = -1)P(S_{ij} = -1|\bm{x}_{i}, \bm{x}_{j})\\
&= (1 - {\rho}_{+1}(\bm{x}_{i}, \bm{x}_{j}))\eta(\bm{x}_{i}, \bm{x}_{j})\\
 &\quad + {\rho}_{-1}(\bm{x}_{i}, \bm{x}_{j})(1-\eta(\bm{x}_{i}, \bm{x}_{j})),
\end{aligned}
\end{equation}
where ${\rho}_{{S}_{ij}}(\bm{x}_{i}, \bm{x}_{j}) = P(\tilde{S}_{ij} = -{S}_{ij}|\bm{x}_{i}, \bm{x}_{j}, {S}_{ij})$ denotes the
flip rate between true labels and noisy labels on given data pair $(\bm{x}_{i}, \bm{x}_{j})$ and their label ${S}_{ij}$.
If we know the values of  ${\rho}_{{S}_{ij}}(\bm{x}_{i}, \bm{x}_{j})$ and $\tilde{\eta}(\bm{x}_{i}, \bm{x}_{j})$,
the values of  $\eta(\bm{x}_{i}, \bm{x}_{j}) $ can be easily inferred. However, the values of ${\rho}_{{S}_{ij}}(\bm{x}_{i}, \bm{x}_{j})$ are unknown.
From Eq.~\ref{equ: 3} we can further get that, when the flip rates
${\rho}_{+1}(\bm{x}_{i}, \bm{x}_{j})$ and ${\rho}_{-1}(\bm{x}_{i}, \bm{x}_{j})$ are relatively small,
if $\tilde{\eta}(\bm{x}_{i}, \bm{x}_{j})$ is large, $\eta(\bm{x}_{i}, \bm{x}_{j}) $ should also be large,
and vice versa. In the following subsection, we show that it is possible to infer whether $\eta(\bm{x}_{i}, \bm{x}_{j})$ is smaller or larger than $0.5$ based on some weak
information\footnote{As many of the labels in $\tilde{S}$ are correct, we show later that upper bounds for ${\rho}_{{S}_{ij}}(\bm{x}_{i}, \bm{x}_{j})$ can be easily obtained.} of ${\rho}_{{S}_{ij}}(\bm{x}_{i}, \bm{x}_{j})$,
which means we
may potentially achieve the reliable labels for some data pairs.
We define those data pairs of which reliable labels can be recovered from $\tilde{S}$ as distilled data pairs.

 In the following subsection, we theoretically prove that distilled data pairs can be extracted under a mild assumption.
And  we further provide a method to collect distilled data pairs automatically.

\subsection{Collecting Distilled Data Pairs Automatically}
To collect distilled data pairs, we first give the following assumption.
\begin{assumption} For any data pairs $\{(\bm{x}_{i}, \bm{x}_{j}), i, j = 1,...N\}$, we have
\label{asum: 1}
\begin{equation}
 \label{equ: 4}
\begin{aligned}
 0\le{\rho}_{+1}(\bm{x}_{i}, \bm{x}_{j}) + {\rho}_{-1}(\bm{x}_{i}, \bm{x}_{j}) \le 1.
 \end{aligned}
\end{equation}
\end{assumption}
This assumption implies that label noise is not too heavy. Note that, if the number of correctly
labeled data pairs is considered larger than that of
mislabeled data pairs, the flip rate ${\rho}_{{S}_{ij}}(\bm{x}_{i}, \bm{x}_{j})$ will be bounded by $0.5$. We can
see that Assumption~\ref{asum: 1} is much weaker than the above assumption.

It is hard to prove that
the noisy labels constructed by exploiting local structures satisfy Assumption~\ref{asum: 1}. However, the experimental results on three widely used benchmark datasets empirically verify that
the assumption applies well to the constructed noisy labels. In the rest of this paper, we always suppose  Assumption \ref{asum: 1} holds.

 We then extend the noisy label learning techniques in~\cite{cheng2017learning} to pairwise labels and present the following key theorem, which gives the
basic criteria to collect distilled data pairs.
%

\begin{theorem}
\label{the: 1}
For any data pairs $\{(\bm{x}_{i}, \bm{x}_{j}), i, j = 1,...N\}$, we have

if $\tilde{\eta}(\bm{x}_{i}, \bm{x}_{j}) < \frac{1-{\rho}_{+1}(\bm{x}_{i}, \bm{x}_{j})}{2}$, then $\{(\bm{x}_{i}, \bm{x}_{j}),s_{ij} = -1\}$ is a distilled data pair;

if $\tilde{\eta}(\bm{x}_{i}, \bm{x}_{j}) > \frac{1+{\rho}_{-1}(\bm{x}_{i}, \bm{x}_{j})}{2}$, then $\{(\bm{x}_{i}, \bm{x}_{j}),s_{ij} = +1\}$ is a distilled data pair.
\end{theorem}
\begin{proof}
 According to Eq.~\ref{equ: 3}, for any data pair $\{(\bm{x}_{i}, \bm{x}_{j})|{\eta}(\bm{x}_{i}, \bm{x}_{j})\ge0.5,i,j = 1,...,N\}$, we have
\begin{equation}
\label{equ: 5}
\begin{aligned}
&\tilde{\eta}(\bm{x}_{i}, \bm{x}_{j}) = (1 - {\rho}_{+1}(\bm{x}_{i}, \bm{x}_{j}))\eta(\bm{x}_{i}, \bm{x}_{j})\\
 &\quad + {\rho}_{-1}(\bm{x}_{i}, \bm{x}_{j})(1-\eta(\bm{x}_{i}, \bm{x}_{j}))\\
& =  \eta(\bm{x}_{i}, \bm{x}_{j})(1 - {\rho}_{+1}(\bm{x}_{i}, \bm{x}_{j})- {\rho}_{-1}(\bm{x}_{i}, \bm{x}_{j}))\\
&\quad + {\rho}_{-1}(\bm{x}_{i}, \bm{x}_{j})\\
&\ge\frac{1 - {\rho}_{+1}(\bm{x}_{i}, \bm{x}_{j}) + {\rho}_{-1}(\bm{x}_{i}, \bm{x}_{j})}{2}\\
&\ge\frac{1 - {\rho}_{+1}(\bm{x}_{i}, \bm{x}_{j})}{2}.
\end{aligned}
\end{equation}
The first inequality holds since ${\eta}(\bm{x}_{i}, \bm{x}_{j})\ge 0.5$ and ${\rho}_{+1}(\bm{x}_{i}, \bm{x}_{j}) + {\rho}_{-1}(\bm{x}_{i}, \bm{x}_{j})\le 1$.
Based on Eq.~\ref{equ: 5}, we have
$${\eta}(\bm{x}_{i}, \bm{x}_{j})\ge 0.5 \Rightarrow \tilde{\eta}(\bm{x}_{i}, \bm{x}_{j}) \ge\frac{1 - {\rho}_{+1}(\bm{x}_{i}, \bm{x}_{j})}{2},$$
which implies that
 $$ \tilde{\eta}(\bm{x}_{i}, \bm{x}_{j})< \frac{1 - {\rho}_{+1}(\bm{x}_{i}, \bm{x}_{j})}{2} \Rightarrow {\eta}(\bm{x}_{i}, \bm{x}_{j})<0.5.$$
Combining this result with Eq.~\ref{equ: 1}, we can label  data pair $(\bm{x}_{i}, \bm{x}_{j})$ with $S_{ij} = -1$, if $\tilde{\eta}(\bm{x}_{i}, \bm{x}_{j})< \frac{1 - {\rho}_{+1}(\bm{x}_{i}, \bm{x}_{j})}{2}$. Similarly, we can prove that data pairs $(\bm{x}_{i}, \bm{x}_{j})$
with $\tilde{\eta}(\bm{x}_{i}, \bm{x}_{j}) > \frac{1+{\rho}_{-1}(\bm{x}_{i}, \bm{x}_{j})}{2}$ can be labeled with $S_{ij} = +1$.
\end{proof}
The trade-off for selecting distilled data pairs
is the need of estimating the conditional probability $\tilde{\eta}$ and the flip rate ${\rho}_{{S}_{ij}}(\bm{x}_{i}, \bm{x}_{j})$.
To estimate  $\tilde{\eta}$, we adopt a probabilistic classification method. Specifically, we design
a deep network to map data pairs to probabilities. Since this objective is similar to the hash code learning, we explore the same architecture
for the estimation of $\tilde{\eta}$ and hash code learning. The detailed description of this deep network will be presented in the next subsection.

For the estimation of the flip rate ${\rho}_{{S}_{ij}}(\bm{x}_{i}, \bm{x}_{j})$, most existing works~\cite{liu2016classification,DBLP:conf/uai/2017} assume the noise to be label- and instance-independent or instance-independent. While
in our method, the flip rate should be label- and instance-dependent, so most existing methods
are not suitable for the current problem. Considering the difficulty to directly estimate the flip rate, we alternatively propose a
method to obtain an upper bound. Formally, we give the following proposition.

\begin{proposition}
 \label{pro: 1}
 Given the conditional probability $\tilde{\eta}(\bm{x}_{i}, \bm{x}_{j})$,
the following inequations holds
 \begin{equation}
  \label{equ: 6}
 \begin{aligned}
 &{\rho}_{-1}(\bm{x}_{i}, \bm{x}_{j})\le \tilde{\eta}(\bm{x}_{i}, \bm{x}_{j}),\\
  &{\rho}_{+1}(\bm{x}_{i}, \bm{x}_{j})\le 1-\tilde{\eta}(\bm{x}_{i}, \bm{x}_{j}).
  \end{aligned}
 \end{equation}
\end{proposition}

\begin{proof}
  According to Eq.~\ref{equ: 3}, we can get
  \begin{equation}
   \label{equ: 7}
  \begin{aligned}
  &\tilde{\eta}(\bm{x}_{i}, \bm{x}_{j}) = (1 - {\rho}_{+1}(\bm{x}_{i}, \bm{x}_{j}))\eta(\bm{x}_{i}, \bm{x}_{j})\\
 &\quad + {\rho}_{-1}(\bm{x}_{i}, \bm{x}_{j})(1-\eta(\bm{x}_{i}, \bm{x}_{j}))\\
& =  \eta(\bm{x}_{i}, \bm{x}_{j})(1 - {\rho}_{+1}(\bm{x}_{i}, \bm{x}_{j})- {\rho}_{-1}(\bm{x}_{i}, \bm{x}_{j}))\\
&\quad + {\rho}_{-1}(\bm{x}_{i}, \bm{x}_{j}) \ge {\rho}_{-1}(\bm{x}_{i}, \bm{x}_{j}).
\end{aligned}
\end{equation}
The inequality holds because  ${\rho}_{+1}(\bm{x}_{i}, \bm{x}_{j}) + {\rho}_{-1}(\bm{x}_{i}, \bm{x}_{j})\le 1$. Similarly, it gives
${\rho}_{+1}(\bm{x}_{i}, \bm{x}_{j})\le 1-\tilde{\eta}(\bm{x}_{i}, \bm{x}_{j})$.
\end{proof}

However, if we directly combine Proposition~\ref{pro: 1} and Theorem~\ref{the: 1}, no distilled data pairs can be selected. So, here we further
assume  the flip rate to be local invariant, and thus obtain the flip rate upper bounds as
\begin{equation}
\label{equ: 8}
 \begin{aligned}
   &{\rho}_{-1max}(\bm{x}_{i}, \bm{x}_{j}) \\
  &= min\{ \tilde{\eta}(\bm{x}_{k}, \bm{x}_{l}) | ,\bm{x}_{k}\in nn_{o}(\bm{x}_{i}), \bm{x}_{l}\in nn_{o}(\bm{x}_{j})\},\\
  &{\rho}_{+1max}(\bm{x}_{i}, \bm{x}_{j}) \\
  &= min\{ (1- \tilde{\eta}(\bm{x}_{k}, \bm{x}_{l})) | ,\bm{x}_{k}\in nn_{o}(\bm{x}_{i}), \bm{x}_{l}\in nn_{o}(\bm{x}_{j})\},
  \end{aligned}
 \end{equation}
where $nn_{o}(\bm{x}_{i})$ indicates the set of top $o$ nearest neighbors for $\bm{x}_{i}$.

With the flip rate upper bounds ${\rho}_{+1max}(\bm{x}_{i}, \bm{x}_{j})$ and ${\rho}_{-1max}(\bm{x}_{i}, \bm{x}_{j})$, we have
\begin{equation}
\begin{aligned}
\label{equ: 81}
 \frac{1-{\rho}_{+max}(\bm{x}_{i}, \bm{x}_{j})}{2}\le \frac{1-{\rho}_{+1}(\bm{x}_{i}, \bm{x}_{j})}{2}\\
 \frac{1+{\rho}_{-max}(\bm{x}_{i}, \bm{x}_{j})}{2} \ge \frac{1+{\rho}_{-1}(\bm{x}_{i}, \bm{x}_{j})}{2}.
 \end{aligned}
\end{equation}
The conditional probability $\tilde{\eta}(\bm{x}_{i}, \bm{x}_{j})$ can be estimated by the adopted deep networks, and
the noisy rate upper bound can be acquired with Eq.~\ref{equ: 8}. Combining these results with Eq.~\ref{equ: 81} and Theorem~\ref{the: 1},
we can find that the distilled data pairs can be successfully collected by
picking out every pairs $(\bm{x}_{i}, \bm{x}_{j})$ that satisfy
$\tilde{\eta}(\bm{x}_{i}, \bm{x}_{j}) > \frac{1+{\rho}_{-1max}(\bm{x}_{i}, \bm{x}_{j})}{2}$
and assigning label $S_{ij} = +1$,
and picking out every pairs $(\bm{x}_{i}, \bm{x}_{j})$ that satisfy
$\tilde{\eta}(\bm{x}_{i}, \bm{x}_{j}) < \frac{1-{\rho}_{+1max}(\bm{x}_{i}, \bm{x}_{j})}{2}$ and
assigning label $S_{ij} = -1$. The distilled data pair set can be represented as $\{(\bm{x}_{i}, \bm{x}_{j}, {S}_{ij}), i,j = 1,...m\}$, where
$m$ is the number of distilled data pairs.

After obtaining the distilled data pair set, we can perform hash code learning, which is similar to the learning process for supervised hashing. Specifically,
we adopt a Bayesian learning framework, which is elaborated in the following subsection.
\subsection{Bayesian Learning Framework}
In this subsection, we propose a Bayesian learning framework
to perform deep hashing learning and also estimate the conditional probability $\tilde{\eta}(\bm{x}_{i}, \bm{x}_{j})$.
We first introduce the framework for hash code learning
and then show how to apply it for the estimation of  $\tilde{\eta}(\bm{x}_{i}, \bm{x}_{j})$. 

By representing the hash codes for distilled data as $\bm{B} = [\bm{b}_{1},...,\bm{b}_
{m}]$, the
Maximum Likelihood (ML) estimation
of the hash codes  can be defined as:
\begin{equation}
 \label{equ: 9}
\log{P(S|\bm{B})} = \frac{1}{m^2}\sum_{i =1}^{m} \sum_{j =1}^{m}\log P({S}_{ij}|\bm{b}_{i},\bm{b}_{j}),
\end{equation}
where  $P({S}_{ij}|\bm{b}_{i},\bm{b}_{j})$ is the conditional probability of similarity label ${S}_{ij}$
given the hash
codes $\bm{b}_{i}$ and $\bm{b}_{j}$, which can be naturally approximated by a pairwise logistic function
\begin{equation}
 \label{equ: 10}
\log P({S}_{ij}|\bm{b}_{i},\bm{b}_{j})=\left\{
\begin{aligned}
&\sigma(\left<\bm{b}_{i}, \bm{b}_{j}\right>)   && {S}_{ij}=1,\\
&1-\sigma(\left<\bm{b}_{i}, \bm{b}_{j}\right>)     && {S}_{ij} = -1,
\end{aligned}
\right.
\end{equation}
where $\sigma(x) = \frac{1}{1+{e}^{-x}}$ is the sigmoid function and $\left<\bm{b}_{i}, \bm{b}_{j}\right>$ denotes the inner product of the
hash codes $\bm{b}_{i}$ and $\bm{b}_{j}$. Here, we adopt the inner product, since as indicated in~\cite{liu2012supervised},
the Hamming distance $dist_{H}(\cdot, \cdot)$ of hash codes can be inferred from the inner product $\left<\cdot, \cdot\right>$  as:
$dist_{H}(\bm{b}_{i}, \bm{b}_{j}) = \frac{1}{2}(K - \left<\bm{b}_{i}, \bm{b}_{j}\right>)$. Hence, the inner product can well reflect the Hamming
distance for binary hash codes.

Similar to logistic regression, we can find that the smaller the Hamming distance $dist_{H}(\bm{b}_{i}, \bm{b}_{j})$
is, the larger the inner product results $\left<\bm{b}_{i}, \bm{b}_{j}\right>$ and the conditional probability
$P(1|\bm{b}_{i},\bm{b}_{j})$ will be. Otherwise, the larger the conditional probability $P(-1|\bm{b}_{i},\bm{b}_{j})$ will be. These results imply
that similar data points will be enforced to have small Hamming distance and dissimilar data points will be enforced to have large
Hamming distance, which are expected for Hamming space similarity search. So,  learning  with Eq.~\ref{equ: 9}, effective hash codes
can be obtained.

After training the model, given a data point, we can obtain its hash codes by directly forward propagating it
through the adopted network, and obtain the final binary codes by the following sign function
\begin{equation}
\begin{aligned}
 \text{sign}(x) = \begin{cases} 1\quad &\text{if } x\ge 0,\\ -1 \quad &\text{if } x < 0.\end{cases}
 \label{eqn: eq1}
 \end{aligned}
\end{equation}
The whole learning algorithm is summarized in Algorithm~\ref{alg: alg1}.

Since this framework maps data pairs to similarity probabilities, we can also use it to estimate the conditional probability. The main difference
lies in that, for hash code learning we use distilled data pairs as inputs,
while for conditional probability estimation, we use the  data
pairs constructed from local structures as inputs.

\begin{algorithm}[!t]
 \SetAlgoNoLine
 \caption{\small DistillHash}
 \textbf{Training Stage}\\
 \BlankLine
 \textbf{Input:}{ Training images $\mathbf{X}$, code length $K$, mini-batch size $t$, hyper-parameters $o$ and $p$.}\\
 \textbf{Procedure:}\\
  1. Construct initial noisy similarity labels with Eq.~\eqref{equ: 2}.\\
  2. Estimate the conditional noisy label probability $\tilde{\eta}(\cdot,\cdot)$ for all training data pairs.\\
  3. Estimate the flip rate upper bounds for all training data pairs with Eq.~\eqref{equ: 8}.\\
  4. Distill data pairs with Theorem~\ref{the: 1}.\\
  \Repeat{convergence}{
    3.1 Randomly sample $t$ data pairs from the distilled data pair set as inputs.\\
    3.2 Calculate the outputs by forward-propagation through the adopted networks.\\
    3.3 Update parameters of the  network by minimizing Eq.~\eqref{equ: 9}.\\}
 \BlankLine
  \textbf{Testing Stage}\\
 \BlankLine
  \textbf{Input:}{ Image query $\mathbf{q}_i$, parameters for the adopted network.}\\
  \textbf{Procedure:}\\
  1. Calculate the output of the neural network by directly forward-propagating the input images.\\
  2. Obtain hash codes with the sign function.
  \label{alg: alg1}
\end{algorithm}

\section{Experiments}
\label{experiments}
We evaluate our method on three popular benchmark datasets, \textbf{FLICKR25K}, \textbf{NUSWIDE}, and \textbf{CIFAR10}, and provide extensive evaluations
to demonstrate its performance. In this section, we first introduce the datasets and then present our experimental results.

\subsection{Datasets}

\textbf{FLICKR25K}~\cite{huiskes08} contains 25,000 images collected from the Flickr website.
Each image is manually annotated with at least one of the 24 unique labels provided.
We randomly select 2,000 images as a test set; the remaining images are used as a retrieval set, from which we randomly select 5,000 images as a training set.
\textbf{NUSWIDE}~\cite{chua2009nus} contains 269,648 images, each of the images is annotated with multiple labels referring to 81 concepts.
The subset containing  the 10 most popular concepts is used here.
We randomly select 5,000 images as a test set; the remaining images are used as a retrieval set, and  10,500 images are randomly selected from the retrieval set as the training set.
\textbf{CIFAR10}~\cite{krizhevsky2009learning} is a popular image dataset containing 60,000 images in 10 classes. For each class, we randomly select
1,000 images as queries and 500  as training images, resulting in a query set containing 10,000 images and a training set made up of 5,000 images.
All  images except for the query set are used as the retrieval set.

\subsection{Baseline Methods}
The proposed method is compared with six state-of-the-art traditional unsupervised hashing methods
(LSH~\cite{andoni2006near}, SH~\cite{weiss2009spectral}, ITQ~\cite{gong2013iterative}, PCAH~\cite{wang2006annosearch}, DSH~\cite{jin2014density},
and SpH~\cite{heo2012spherical}) and three recently proposed deep unsupervised hashing methods (DeeBit~\cite{lin2016learning}, SGH~\cite{dai2017stochastic}, and SSDH~\cite{yang2018semantic}).
%
All the codes for these methods have been kindly provided by
the authors. LSH, SH, ITQ, PCAH, DSH, and SpH are implemented with MATLAB, SGH and SSDH are implemented
with TensorFlow~\cite{abadi2016tensorflow}, and DeepBit is implemented with Caffe~\cite{jia2014caffe}. We use TensorFlow when write
our code, and run the algorithm in a machine with one Titan
X Pascal GPU.

\subsection{Evaluation.}
To evaluate the performance of our proposed method, we adopt three evaluation criteria: mean of average precision (MAP),
topN-precision,  and precision-recall. The first two criteria are based on Hamming ranking, which ranks data points based on their Hamming distances to the query;
for its part, precision-recall is based on hash lookup. More detailed introductions are given as follows.

\textbf{MAP} is one of the most widely-used criteria for evaluating retrieval accuracy. Given a query and a list
of $R$ ranked retrieval results, the average precision (AP) for this query can be computed. MAP is then defined as the average of APs for all queries.
 For all three datasets, we set $R$ as the number of the retrieval set.
\textbf{TopN-precision} is defined as the
average ratio of similar instances among the top $N$ retrieved instances for
all queries in terms of Hamming distance. In our experiments, $N$ is set to 1,000.
\textbf{Precision-recall} reveals the precisions
 at different recall levels and is a good indicator of
 overall performance.
  Typically,
the area under the precision-recall curve is computed. A larger Precision-recall value always indicates better performance.
\subsection{Implementation Details}
To initialize the noisy similarity matrix in Eq.~\eqref{equ: 2}, we select the cosine distance as the distance to measure the local structure of training examples. The threshold
$t_{1}$ and $t_{2}$ are selected as indicated in~\cite{yang2018semantic}.
For the adopted deep networks, we use the VGG16 architecture~\cite{simonyan2014very} and replace
the last fully-connected layer with a new fully-connected layer with $K$ units for hash code learning.
For the estimation of the conditional probability $\tilde{\eta}$, we set the dimensions of the last fully-connected layer as $p$, which is $48$ in our experiments. To obtain the upper bound of the flip rate, we
 set $o$ as $4$. The parameter sensitivity of our algorithm with regard to $o$ and $p$ are
  analyzed in Subsection~\ref{parameter}.
Parameters for the new fully connected layer are learned from scratch,
while parameters for the preceding layers are fine-tuned from the model pre-trained on ImageNet~\cite{deng2009imagenet}.
 We employ the standard stochastic gradient
descent algorithm with 0.9 momentum for optimization, min-batch size is set to 64, and the learning rate is fixed to $10^{-3}$.
Two data points are considered neighbors if they share the same label (for CIFAR10) or share at least one common label (for the multi-label datasets FLICKR25K and NUSWIDE).

For a fair comparison,  we
adopt the deep features extracted from the last fully-connected layer from the VGG16 network pre-trained on ImageNet for all shallow architecture-based baseline methods. These deep features are also used for the construction of $\tilde{S}$. Since VGG16 accepts
images of size $224\times224$ as inputs, we resize all images  to be $224\times224$ before inputting them into the VGG16 network. Random rotation and flipping are also used
for data augmentation.


\begin{table*}[!t]
\small
\newcommand{\tabincell}[2]{\begin{tabular}{@{}#1@{}}#2\end{tabular}}
\centering
\caption{Comparison with baselines in terms of MAP. The best accuracy is shown in boldface.}
\begin{tabular}{ccccccccccccc}
      \toprule
 \multirow{2}{*}{method} &\multicolumn{4}{c}{FLICKR25K}&\multicolumn{4}{c}{NUSWIDE}&\multicolumn{4}{c}{CIFAR10}\\
\cmidrule(l){2-5}\cmidrule(l){6-9}\cmidrule(l){10-13}
& 16 bits & 32 bits & 64 bits & 128 bits & 16 bits & 32 bits & 64 bits & 128 bits & 16 bits & 32 bits & 64 bits & 128 bits\\
 \midrule
 LSH~\cite{andoni2006near}     & 0.5831 & 0.5885 & 0.5933 & 0.6014 & 0.4324 & 0.4411 & 0.4433 & 0.4816 & 0.1319 & 0.1580 & 0.1673 & 0.1794 \\
 SH~\cite{weiss2009spectral}      & 0.5919 & 0.5923 & 0.6016 & 0.6213 & 0.4458 & 0.4537 & 0.4926 & 0.5000 & 0.1605 & 0.1583 & 0.1509 & 0.1538 \\
 ITQ~\cite{gong2013iterative}     & 0.6192 & 0.6318 & 0.6346 & 0.6477 & 0.5283 & 0.5323 & 0.5319 & 0.5424 & 0.1942 & 0.2086 & 0.2151 & 0.2188 \\
 PCAH~\cite{wang2006annosearch}    & 0.6091 & 0.6105 & 0.6033 & 0.6071 & 0.4625 & 0.4531 & 0.4635 & 0.4923 & 0.1432 & 0.1589 & 0.1730 & 0.1835 \\
 DSH~\cite{jin2014density}     & 0.6074 & 0.6121 & 0.6118 & 0.6154 & 0.5200 & 0.5227 & 0.5345 & 0.5370 & 0.1616 & 0.1876 & 0.1918 & 0.2055 \\
 SpH~\cite{heo2012spherical}     & 0.6108 & 0.6029 & 0.6339 & 0.6251 & 0.4532 & 0.4597 & 0.4958 & 0.5127 & 0.1439 & 0.1665 & 0.1783 & 0.1840 \\

 DeepBit~\cite{lin2016learning} & 0.5934 & 0.5933 & 0.6199 & 0.6349 & 0.4542 & 0.4625 & 0.47616 & 0.4923 & 0.2204 & 0.2410 & 0.2521 & 0.2530 \\
 SGH~\cite{dai2017stochastic}     & 0.6162 & 0.6283 & 0.6253 & 0.6206 & 0.4936 & 0.4829 & 0.4865 & 0.4975 & 0.1795 & 0.1827 & 0.1889 & 0.1904 \\
 SSDH~\cite{yang2018semantic}    & 0.6621 & 0.6733 & 0.6732 & 0.6771 & 0.6231 & 0.6294 & 0.6321 & 0.6485 & 0.2568 & 0.2560 & 0.2587 & 0.2601 \\
 DistillHash &\textbf{0.6964} &\textbf{0.7056} & \textbf{0.7075} & \textbf{0.6995}
 &\textbf{0.6667} &\textbf{0.6752} & \textbf{0.6769} & \textbf{0.6747} &\textbf{0.2844} & \textbf{0.2853} & \textbf{0.2867}& \textbf{0.2895}\\
 \hline

\end{tabular}
\label{tab: tab3}
\vspace{-16pt}
\end{table*}
\subsection{Results and Discussion}
We first present the MAP values for all methods with different hash bit lengths, then
draw precision-recall and TopN-precision curves for all methods with  32 and 64 hash code lengths to give a more comprehensive comparison.

Table~\ref{tab: tab3} presents the MAP results for DistillHash and all baseline methods on FLICKR25K, NUSWIDE, and CIFAR10, with hash code numbers varying from 16 to 128.
By comparing the data-independent method LSH
with other data-dependent methods, we can see that data-dependent hashing methods outperform the data-independent hashing method in most cases. This may be because that data-dependent methods
learn hash functions from data, so can better capture the used data structures. By comparing deep hashing methods and no-deep hashing methods, we
find that no-deep hashing methods can surpass deep hashing methods DeepBit and SGH in some cases. This may be because that, without proper supervisory signals, deep hashing methods
cannot fully exploit the representation ability of deep networks, and may achieve unsatisfactory performance by over-fitting to bad local minima.
While, by exploiting  local structures, deep hashing methods (SSDH and DistillHash)  achieve more promising results.

Concretely, from the MAP results, we can see that DistillHash consistently obtains the best
results across different hash bit lengths for all three datasets.
Specifically, compared to one of the best non-deep hashing
methods, i.e, ITQ, we
achieve absolute improvements of $6.89\%$, $13.97\%$, and $7.73\%$ in the average
MAP for different bits on FLICKR25K, NUSWIDE, and
CIFAR10 respectively. Compared to the state-of-the-art
deep hashing method SSDH, we achieve absolute improvements of
$3.08\%$, $4.01\%$, and  $2.86\%$ in average MAP for different bits on the
three datasets respectively.
Note that DeepBit, SGH, SSDH, and DistillHash are both deep hashing methods,
only SSDH and DistillHash can exploit and preserve the similarity of different data points, thus they can achieve better performance than the other two.
Moreover, DistillHash learns more accurate similarity relationships by distilling some data pairs, so can obtain a further performance improvement than SSDH.
\begin{figure*}[!t]
\subfigure[TopN-precision with 16bits]{\label{fig: precision_mir_16}\includegraphics[trim={0 0 0 0},width=0.245\textwidth]{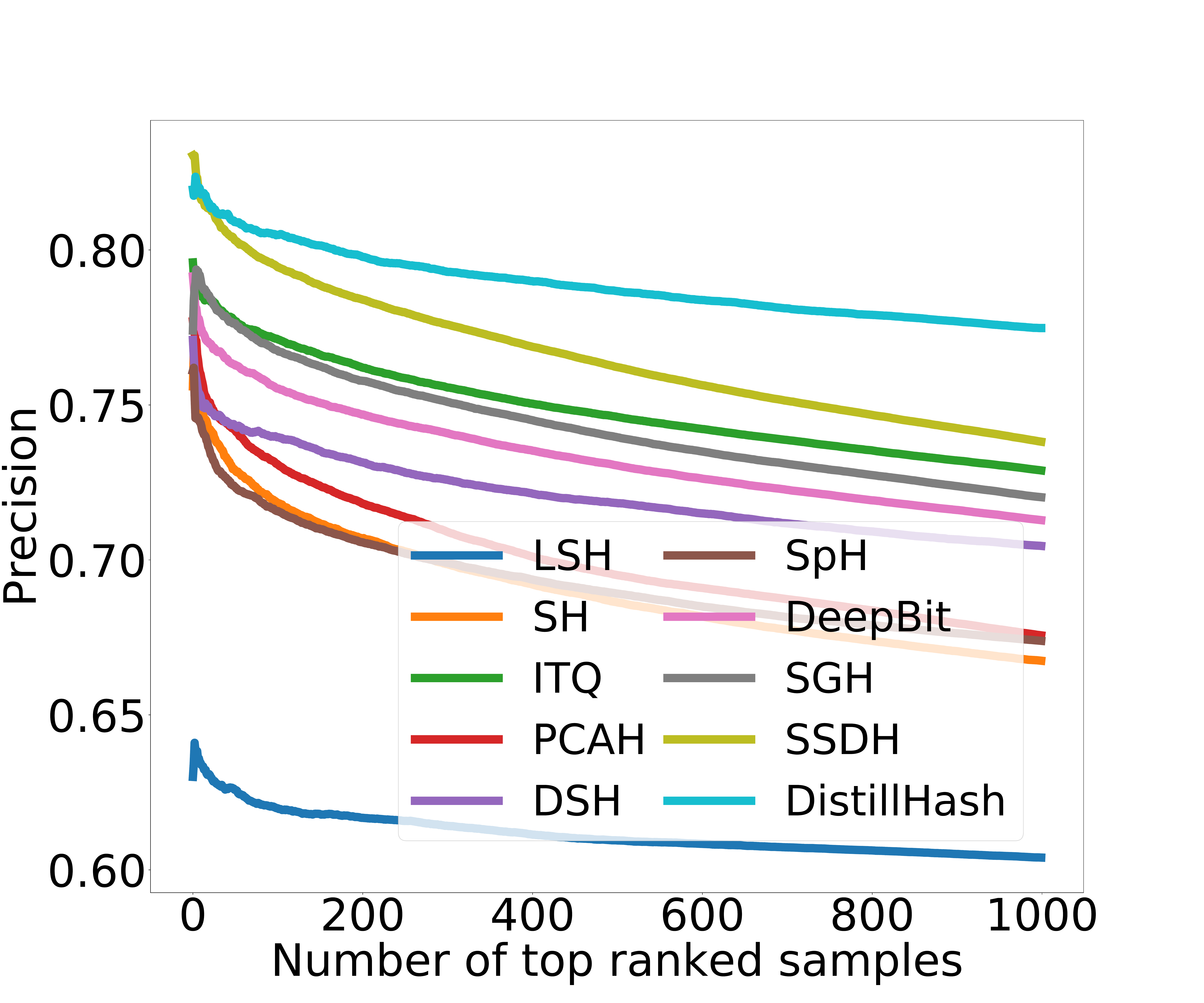}}
\subfigure[TopN-precision with 32bits]{\label{fig: precision_mir_32}\includegraphics[trim={0 0 0 0},width=0.245\textwidth]{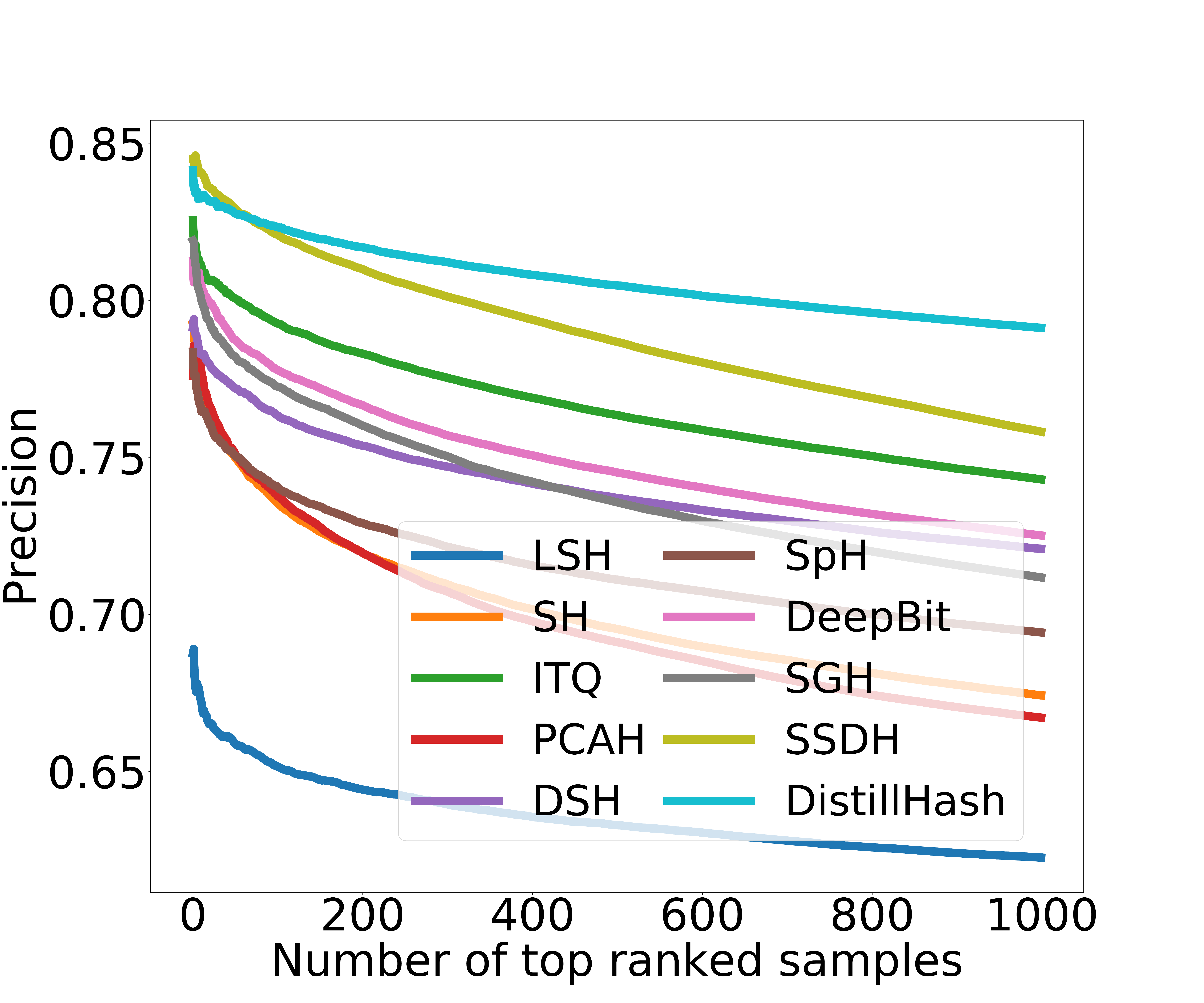}}
\subfigure[Precision-recall with 16bits]{\label{fig: PR_mir_16}\includegraphics[trim={0 0 0 0},width=0.245\textwidth]{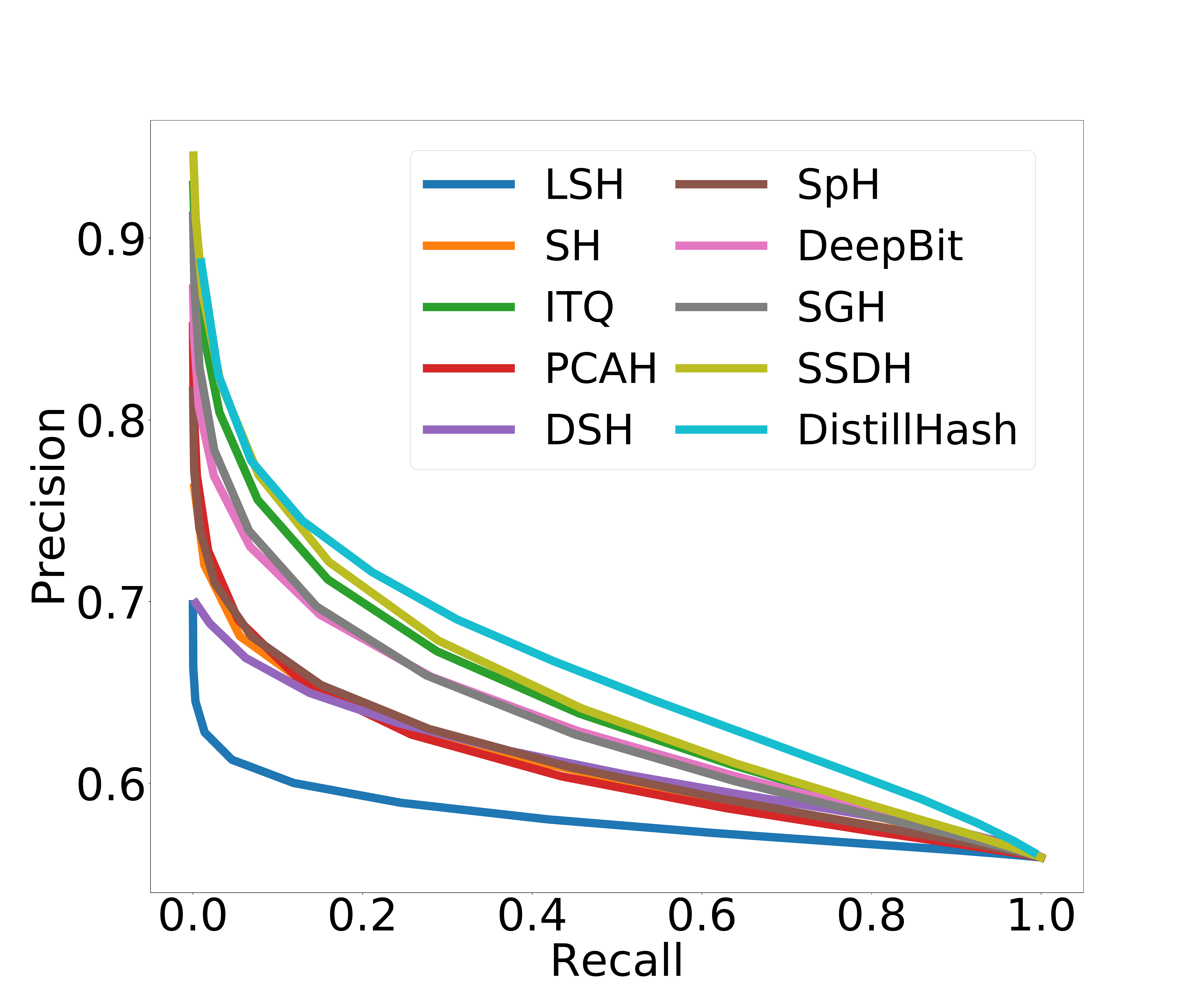}}
\subfigure[Precision-recall with 32bits]{\label{fig: PR_mir_32}\includegraphics[trim={0 0 0 0},width=0.245\textwidth]{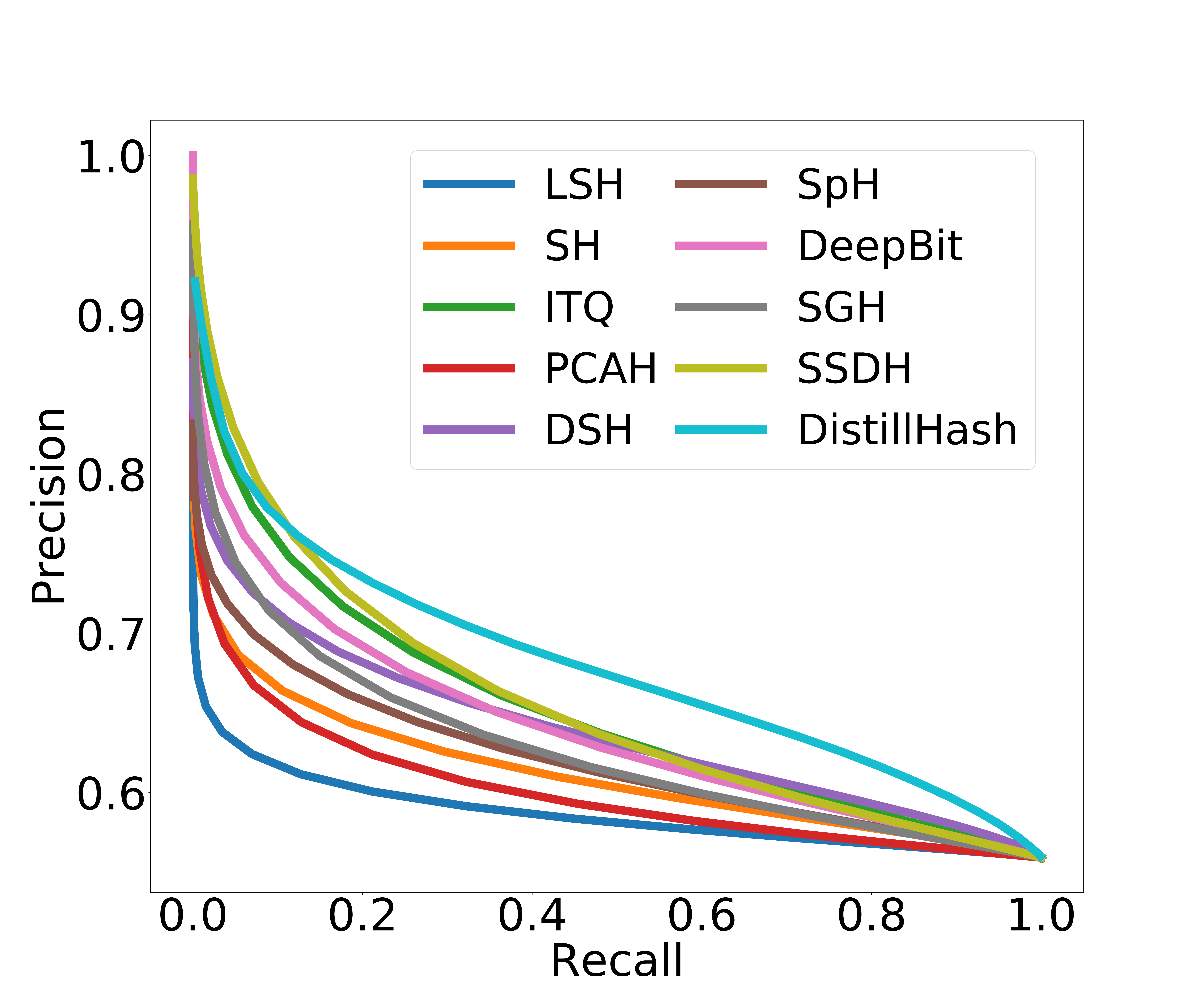}}
\caption{TopN-precision and precision-recall curves on FLICKR25K with 16 and 32 hash bits.}
\label{fig: flickr}
\vspace{-14pt}
\end{figure*}
\begin{figure*}[!t]
\subfigure[TopN-precision with 16bits]{\label{fig: precision_nus_16}\includegraphics[trim={0 0 0 0},width=0.245\textwidth]{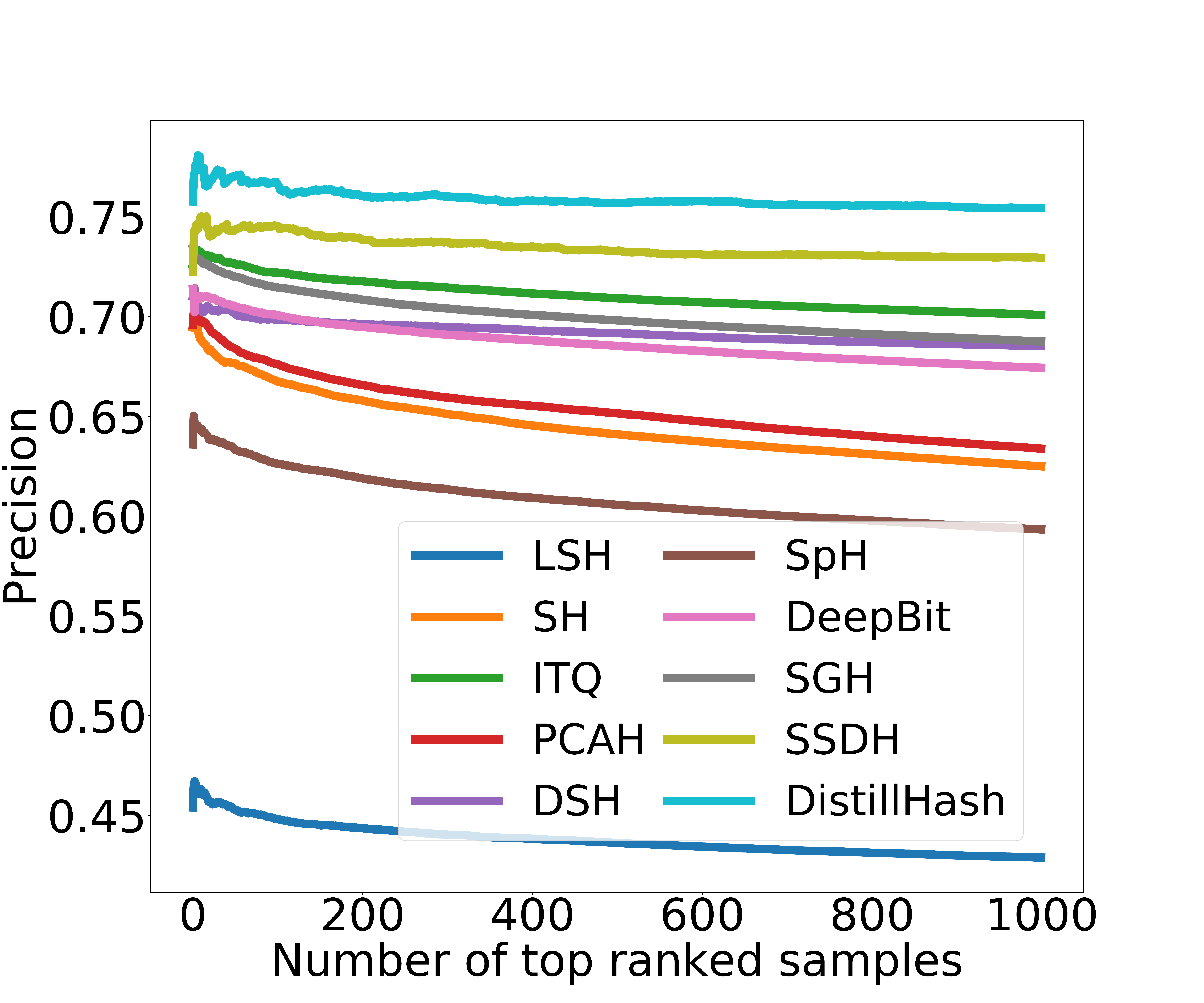}}
\subfigure[TopN-precision with 32bits]{\label{fig: precision_nus_32}\includegraphics[trim={0 0 0 0},width=0.245\textwidth]{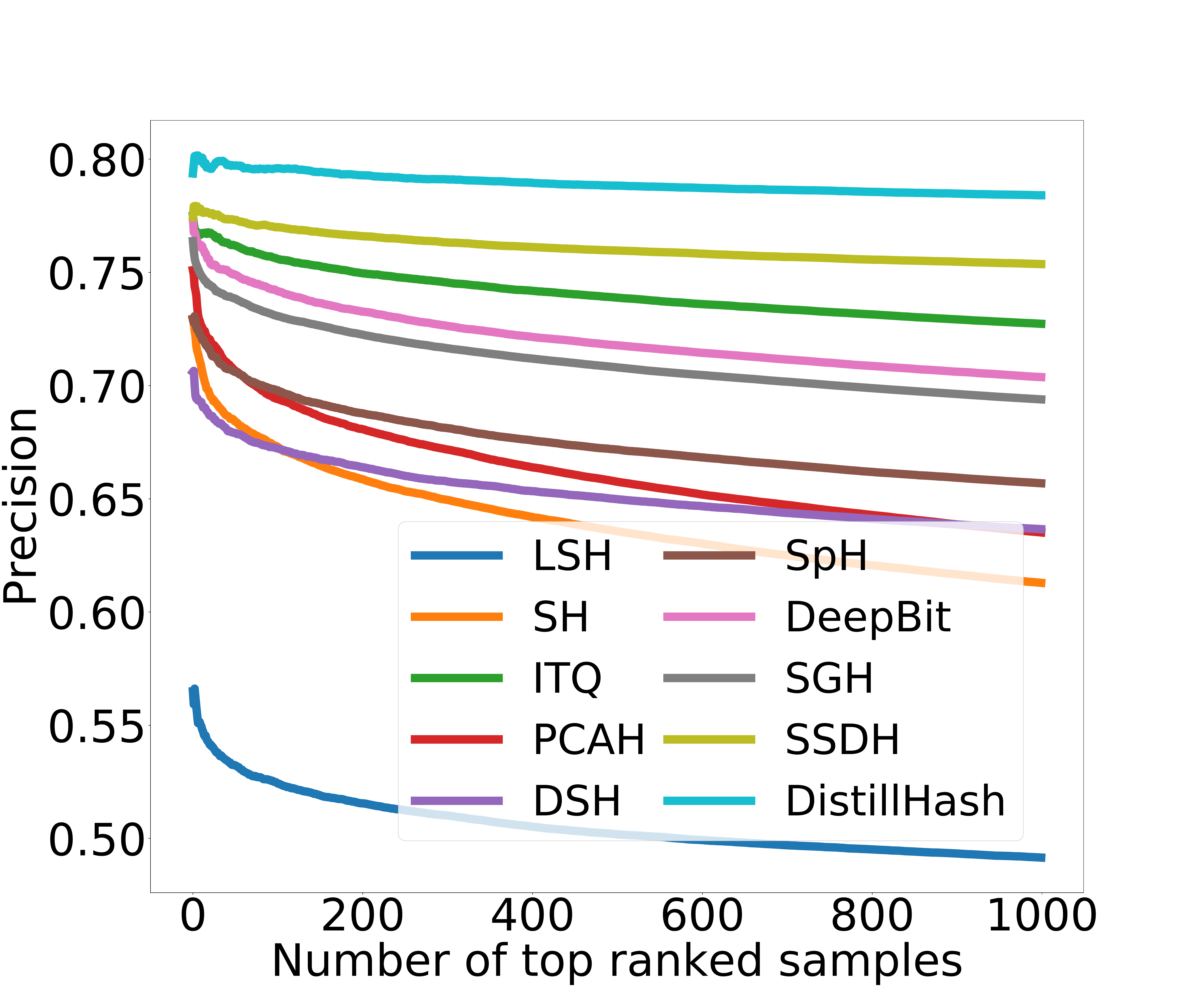}}
\subfigure[Precision-recall with 16bits]{\label{fig: PR_nus_16}\includegraphics[trim={0 0 0 0},width=0.245\textwidth]{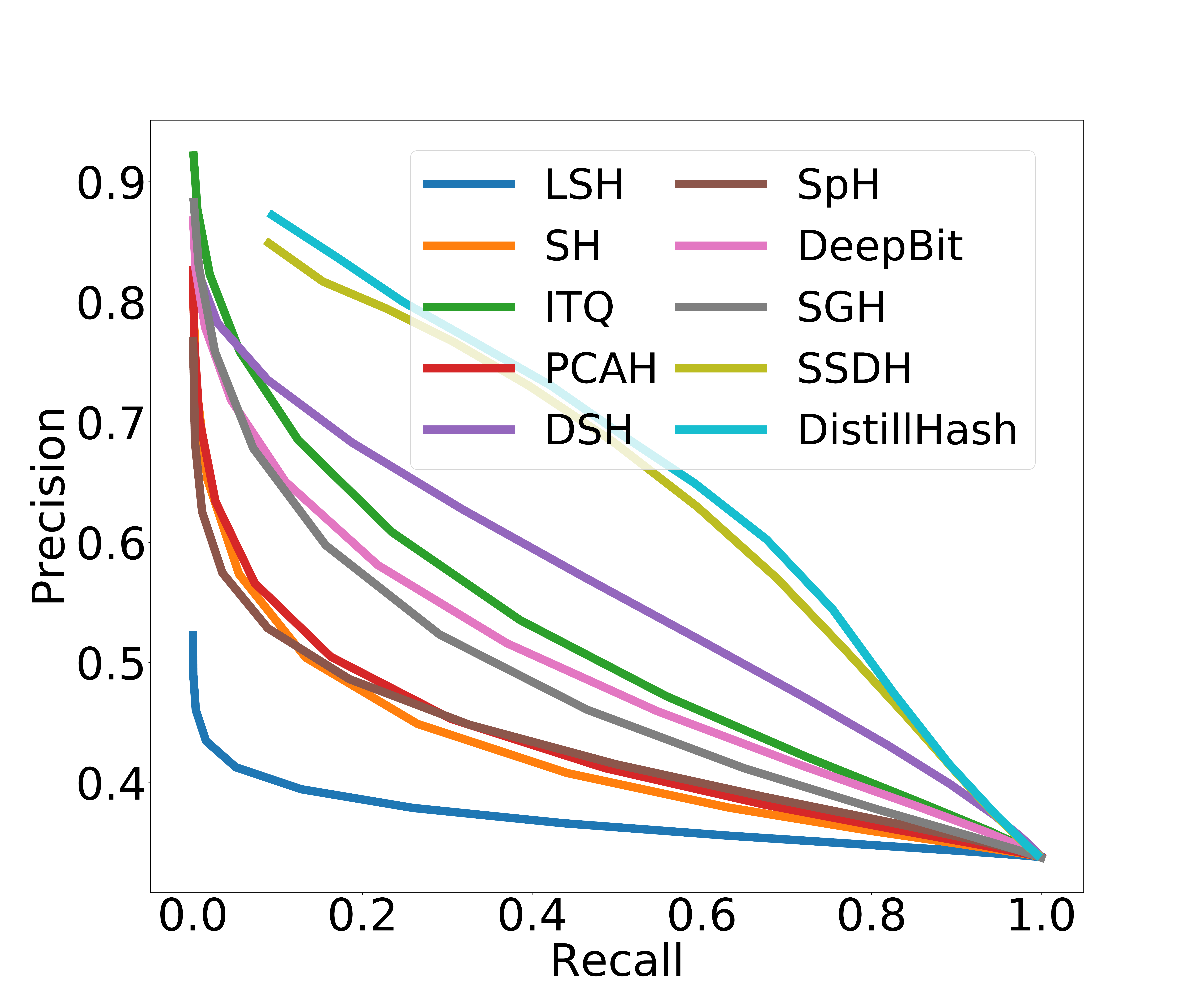}}
\subfigure[Precision-recall with 32bits]{\label{fig: PR_nus_32}\includegraphics[trim={0 0 0 0},width=0.245\textwidth]{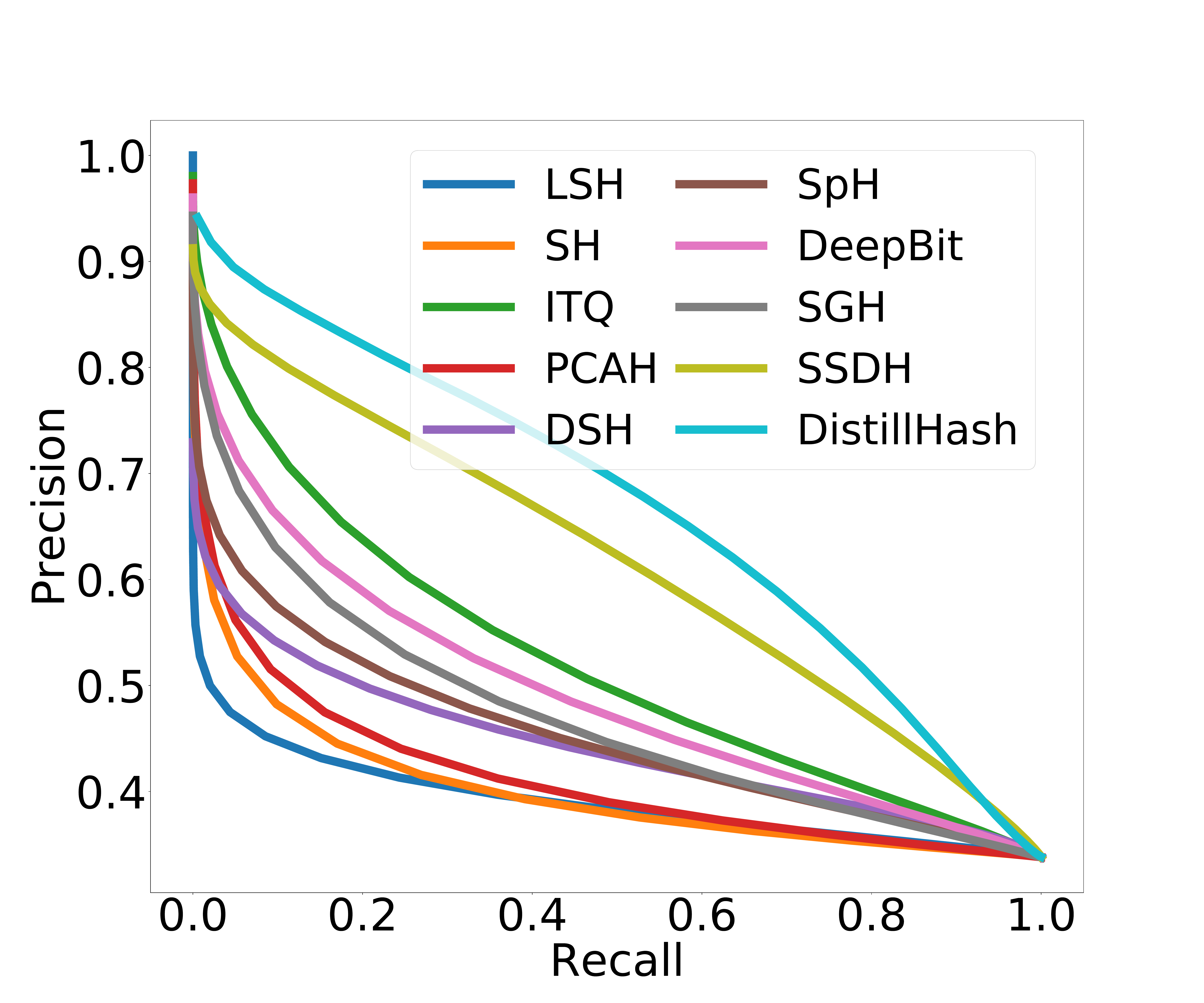}}
\caption{TopN-precision and precision-recall curves on NUSWIDE with 16 and 32 hash bits.}
\label{fig: nuswide}
\vspace{-14pt}
\end{figure*}
\begin{figure*}[!t]
\subfigure[TopN-precision with 16bits]{\label{fig: precision_cifar10_16}\includegraphics[trim={0 0 0 0},width=0.245\textwidth]{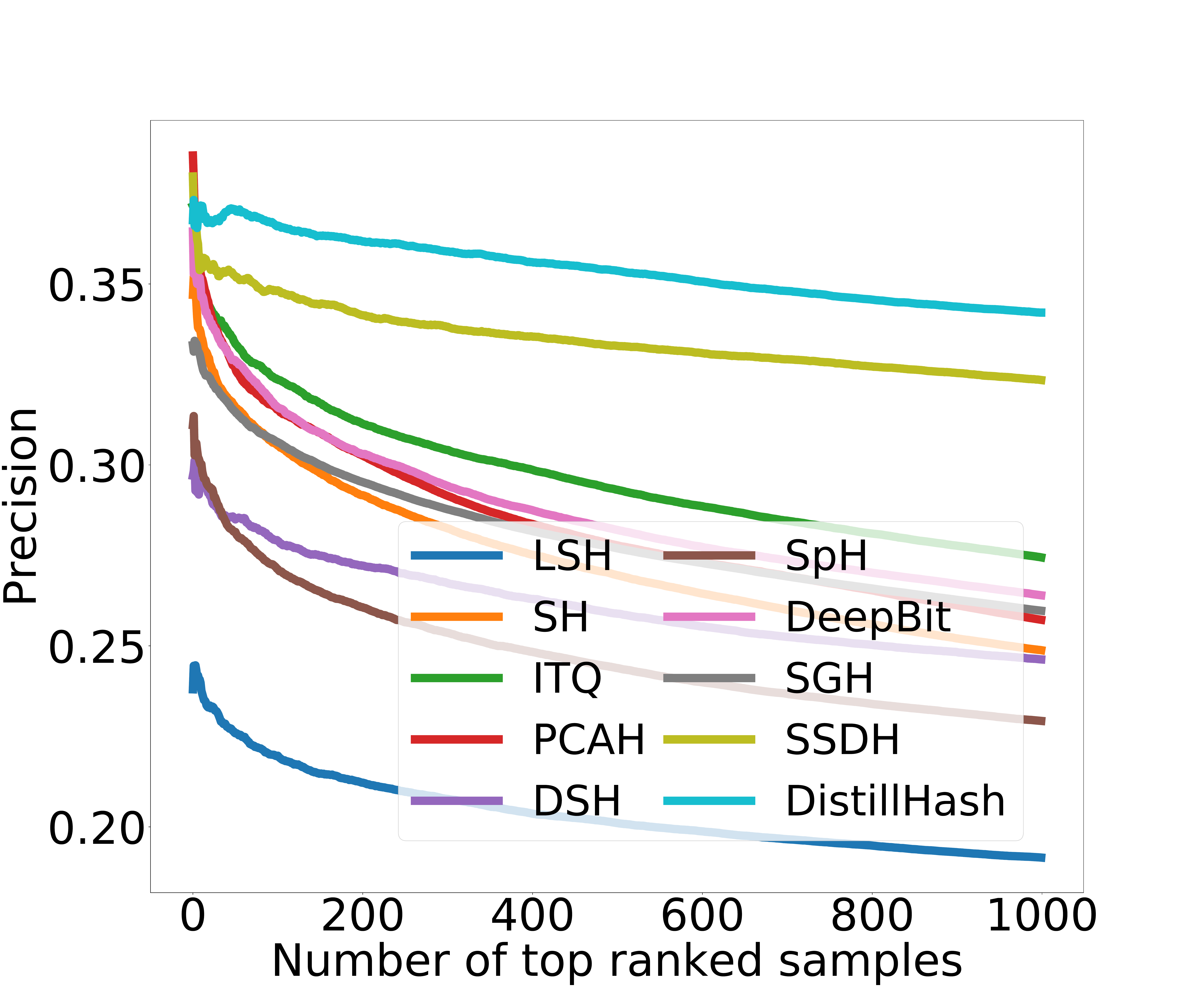}}
\subfigure[TopN-precision with 32bits]{\label{fig: precision_cifar10_32}\includegraphics[trim={0 0 0 0},width=0.245\textwidth]{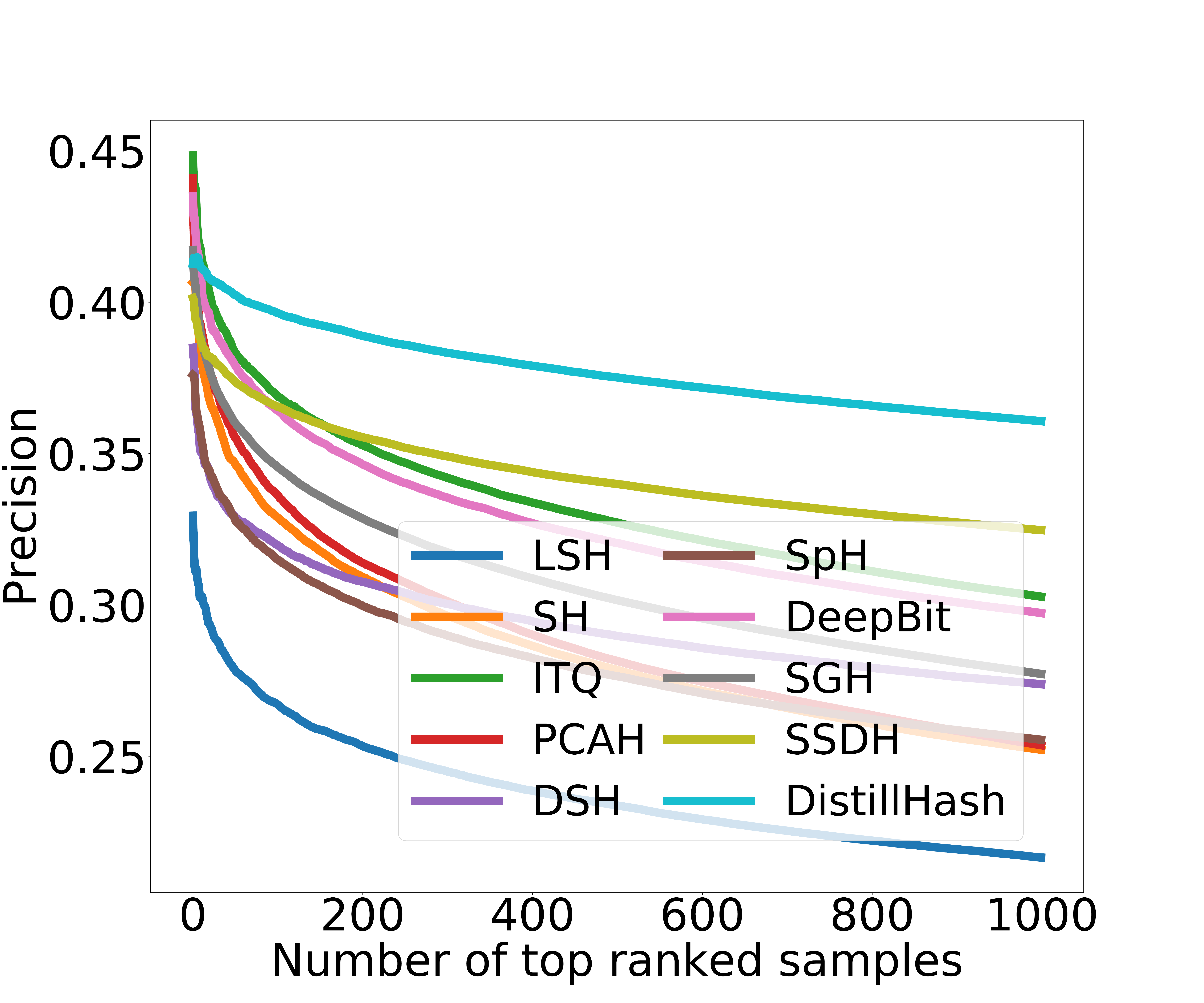}}
\subfigure[Precision-recall with 16bits]{\label{fig: PR_cifar10_16}\includegraphics[trim={0 0 0 0},width=0.245\textwidth]{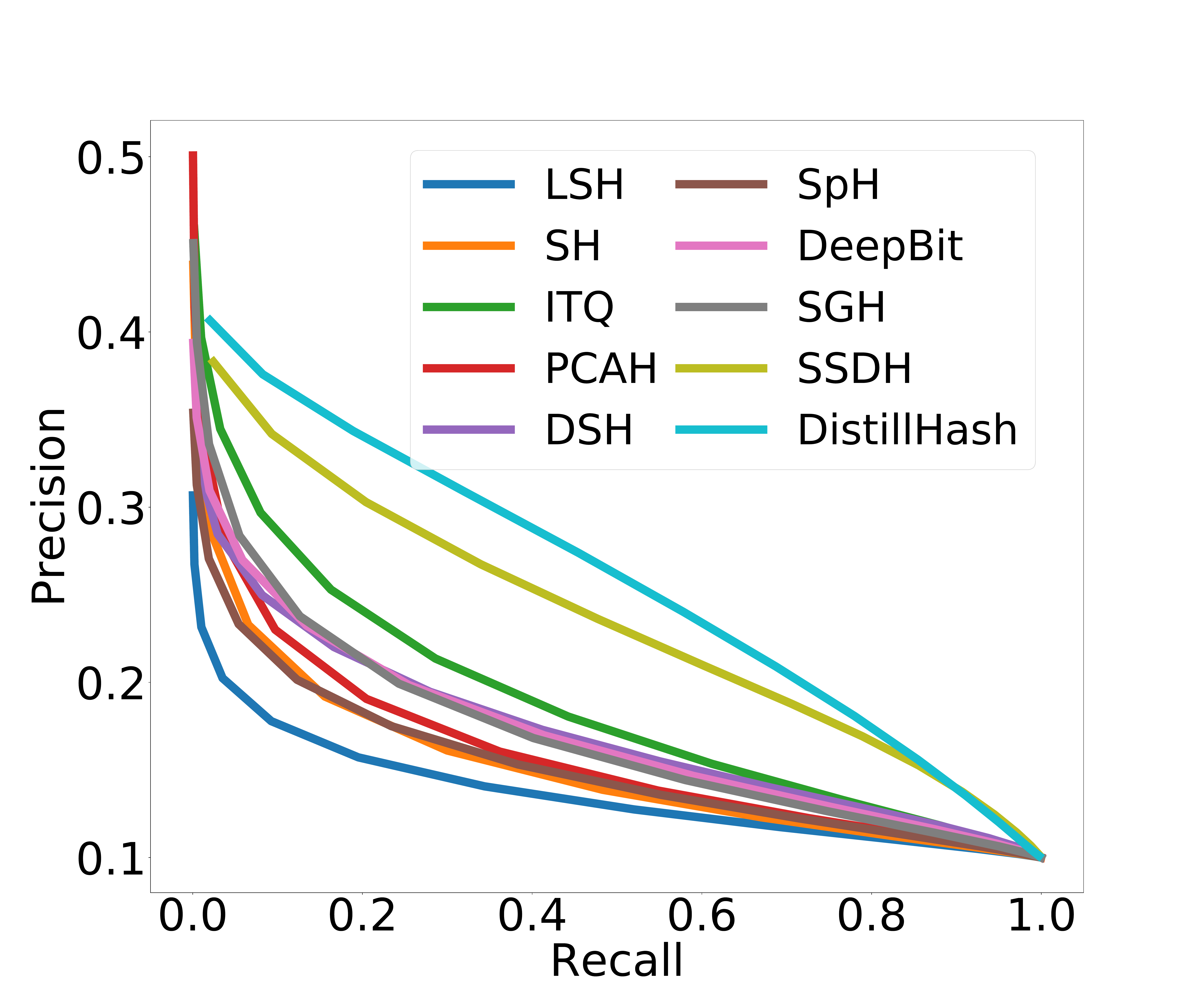}}
\subfigure[Precision-recall with 32bits]{\label{fig: PR_cifar10_32}\includegraphics[trim={0 0 0 0},width=0.245\textwidth]{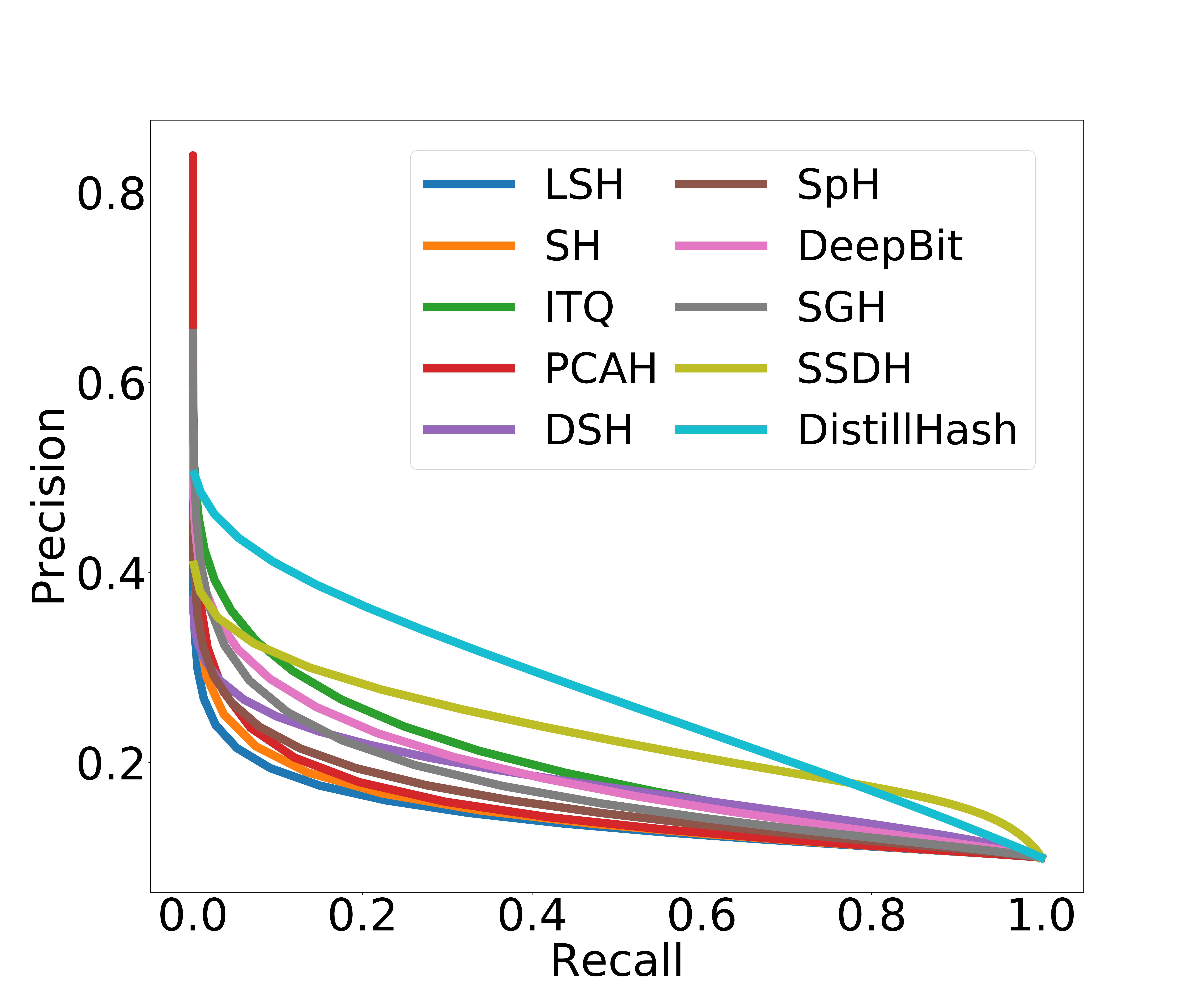}}
\caption{TopN-precision and precision-recall curves on CIFAR10 with 16 and 32 hash bits.}
\label{fig: cifar10}
\vspace{-14pt}
\end{figure*}
%

The left two subfigures of Figure~\ref{fig: flickr},~\ref{fig: nuswide}, and~\ref{fig: cifar10} present the TopN-precision curves for all methods on each of
the three datasets with  hash bit lengths of 16 and 32. Consistent with MAP results, we
can observe that DistillHash achieves the best results among
all approaches.
\begin{table*}[!t]
\small
\newcommand{\tabincell}[2]{\begin{tabular}{@{}#1@{}}#2\end{tabular}}
\centering
\caption{MAP results for DistillHash* and DistillHash. The best accuracy is shown in boldface.}
\begin{tabular}{ccccccccccccc}
      \toprule
 \multirow{2}{*}{method} &\multicolumn{4}{c}{FLICKR25K}&\multicolumn{4}{c}{NUSWIDE}&\multicolumn{4}{c}{CIFAR10}\\
\cmidrule(l){2-5}\cmidrule(l){6-9}\cmidrule(l){10-13}
& 16 bits & 32 bits & 64 bits & 128 bits & 16 bits & 32 bits & 64 bits & 128 bits & 16 bits & 32 bits & 64 bits & 128 bits\\
 \midrule
 DistillHash*     & 0.6653 & 0.6633 & 0.6726 & 0.6784 & 0.6322 & 0.6357 & 0.6480 & 0.6451 & 0.2547 & 0.2538 & 0.2573 & 0.2583 \\
 DistillHash{ }{ }      &\textbf{0.6964} &\textbf{0.7056} & \textbf{0.7075} & \textbf{0.6995}
 &\textbf{0.6667} &\textbf{0.6752} & \textbf{0.6769} & \textbf{0.6747} &\textbf{0.2844} & \textbf{0.2853} & \textbf{0.2867}& \textbf{0.2895}\\
 \hline
\end{tabular}
\label{tab: ab}
\vspace{-16pt}
\end{table*}
Since MAP values and TopN-precision  curves are both
Hamming ranking-based metrics, an overview of the above analysis reveals that
DistillHash can achieve superior performance for Hamming ranking-based evaluations. Moreover, to illustrate the hash lookup results, we plot
the precision-recall curves for all methods with  hash bit lengths of 16 and 32
in the right two subfigures of Figure~\ref{fig: flickr},~\ref{fig: nuswide}, and~\ref{fig: cifar10}. From the results, we can again observe that DistillHash consistently achieves the best
performance, which further demonstrates the superiority of our proposed method.

\begin{figure}[!t]
\subfigure[FLICKR25K]{\label{fig: loss_flickr}\includegraphics[trim={0 0 0 0},width=0.15\textwidth]{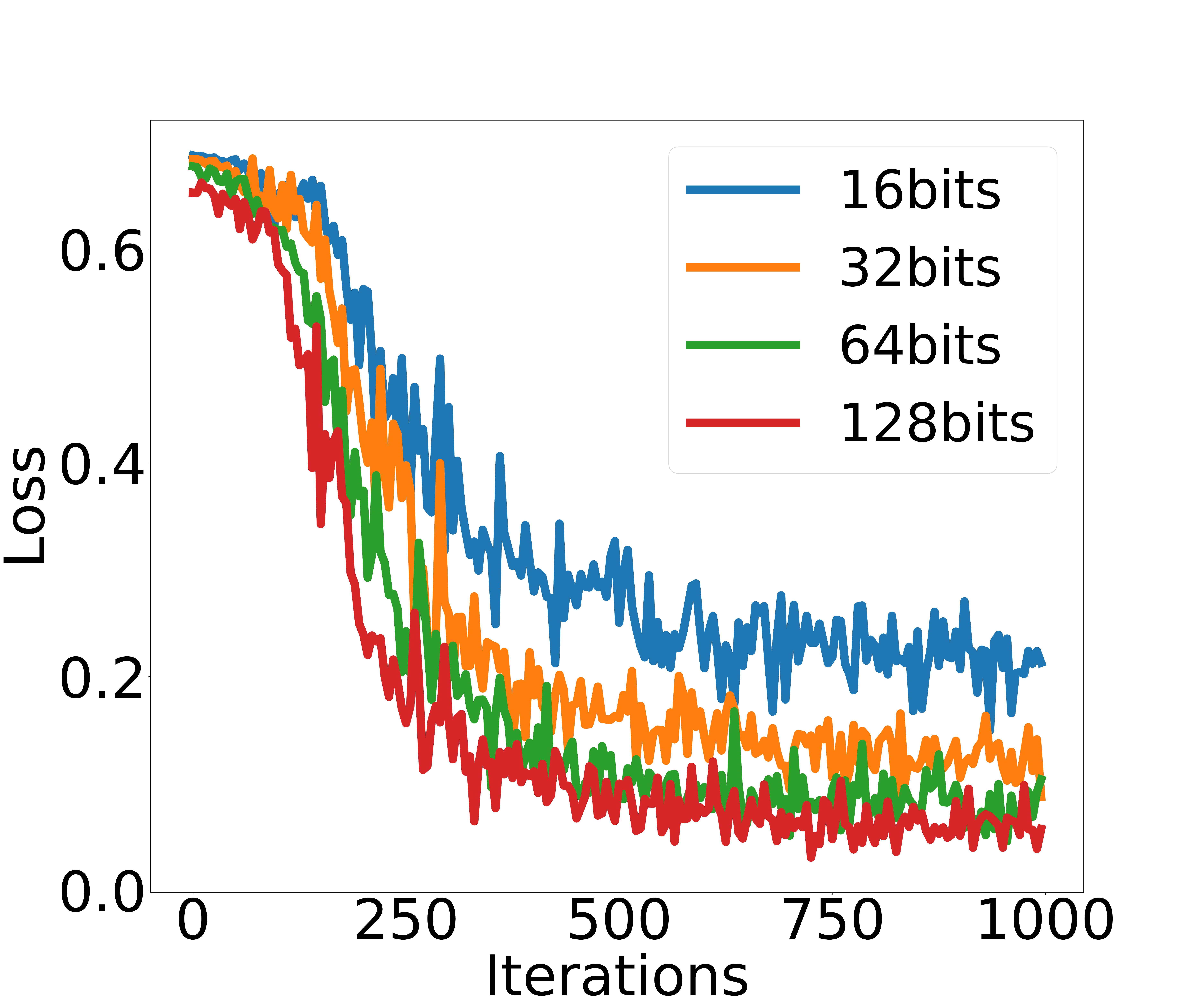}}
\subfigure[NUSWIDE]{\label{fig: loss_nuswide}\includegraphics[trim={0 0 0 0},width=0.15\textwidth]{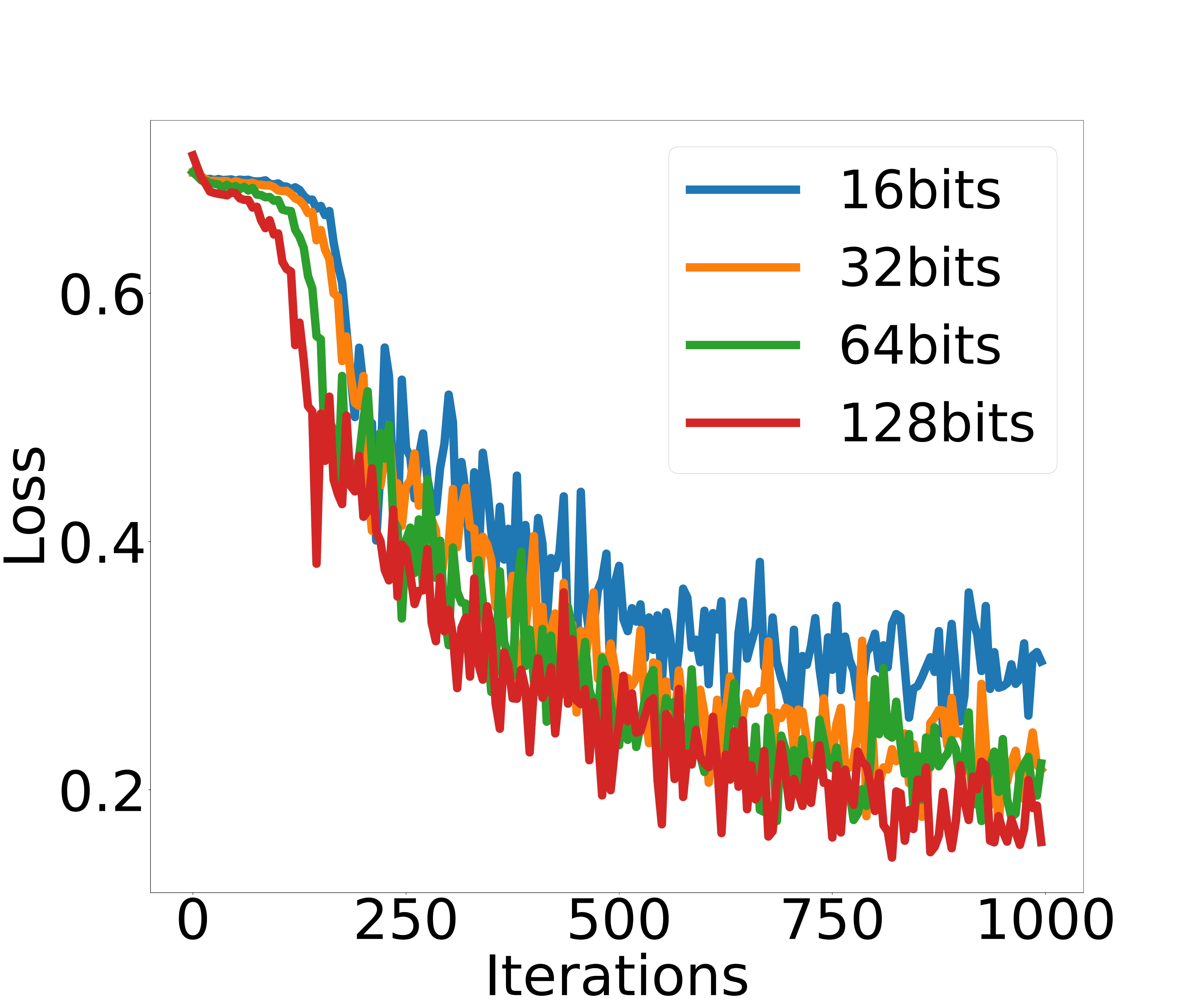}}
\subfigure[CIFAR10]{\label{fig: loss_cifar10}\includegraphics[trim={0 0 0 0},width=0.15\textwidth]{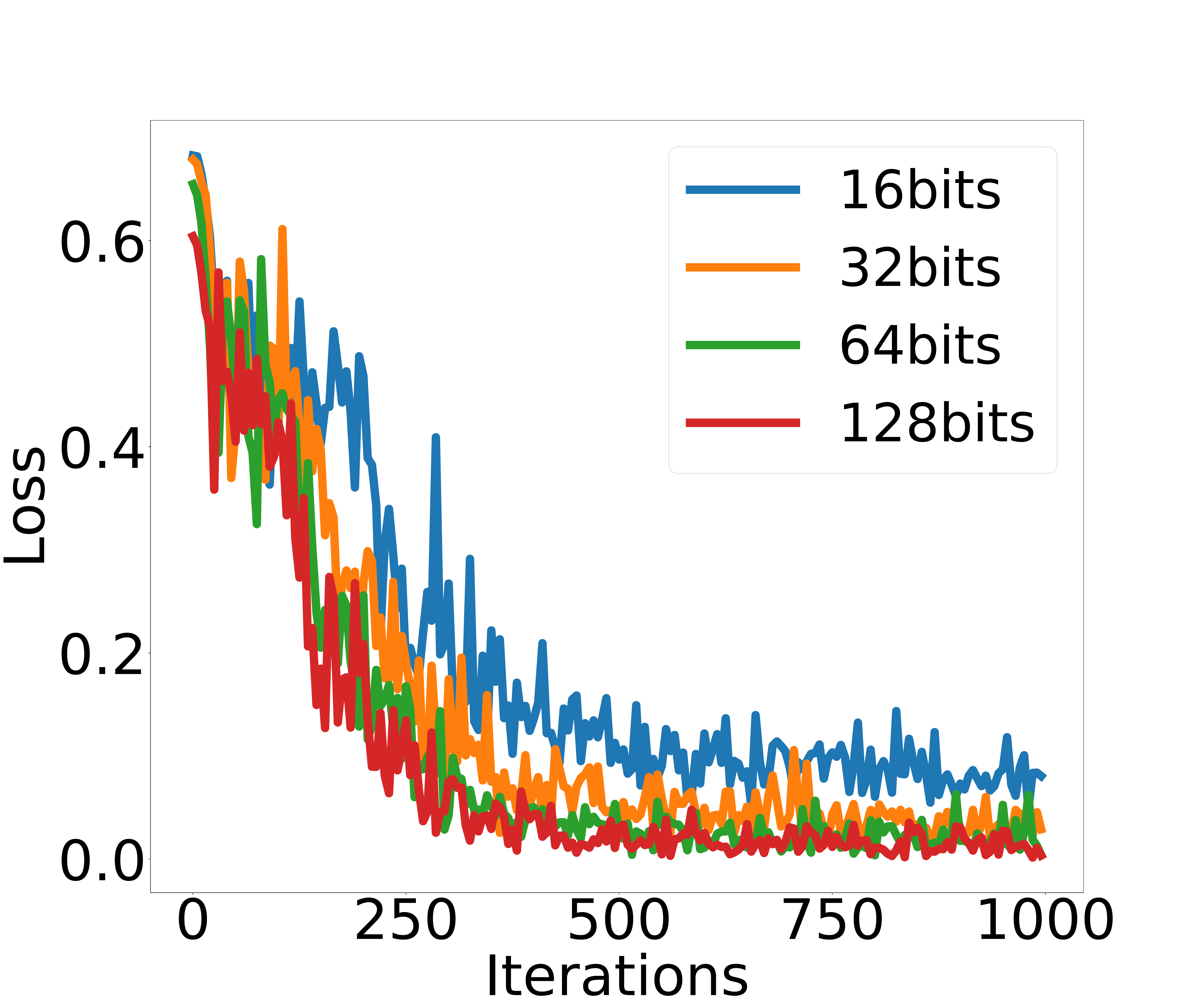}}
\caption{Losses of DistillHash through the training process.}
\label{fig: loss}
\vspace{-12pt}
\end{figure}
To investigate the change of loss values through the training process, we display the loss values in Figure~\ref{fig: loss}. The results reveal that our methods can
converge in all cases within 1,000 iterations.
\subsection{Parameter Sensitivity}
\label{parameter}
We next investigate the sensitivity of hyper-parameters $o$ and $p$. Figure~\ref{fig: 4} shows the effect of these two hyper-parameters on NUSWIDE dataset with hash code lengths of
16, 32, 64, and 128. We first fix $p$ to $48$ and evaluate the MAP  by varying $o$ between $2$ and $20$, the results are presented in Figure~\ref{fig: o}.
 The performance shows that the algorithm is not sensitive to parameter $o$ in the range of $[2, 20]$, and we can set $o$ as any number in the range of $[2, 20]$. In our experiments, we set $o$ as $4$.
Figure~\ref{fig: p} shows the MAP  by varying $p$ between $16$ and $128$ with $o$ fixed to $4$. The performance
of DistillHash first increases and then keeps at a relatively high level. The result is also not sensitive to $p$ in the range of $[32, 128]$
. For other experiments in this paper, we select $p$ as $48$.


\subsection{Ablation Study}
\label{ablation}
In this subsection, we go deeper to study the efficacy of the proposed
distilled data pair learning. More specifically, we investigate DistillHash*,
 a variant of DistillHash with the same Bayesian learning framework but trained with the initial similarity
 label $\tilde{\bm{S}}$.
The MAP results of DistillHash* and DistillHash are shown in Table~\ref{tab: ab}, from which we can see that
DistillHash consistently outperforms DistillHash* by
 margins of $3.11\%$, $4.23\%$, $3.49\%$ and $2.11\%$ for the FLICKR25K dataset,
 $3.45\%$, $3.95\%$, $2.89\%$ and $2.96\%$ for the NUSWIDE dataset, and
 $2.97\%$, $3.15\%$, $2.94\%$ and $3.12\%$ for the CIFAR10 dataset at hash bit lengths of 16, 32, 64, and 128 respectively. Note that the only difference between DistillHash and DistillHash* lies in  that
 DistillHash is trained with distilled data set and DistillHash* is trained with the initial data set. The performance improvements clearly demonstrate the superiority of the proposed distilled data pair learning.
\begin{figure}[!t]
\subfigure[MAP w.r.t. different $o$.]{\label{fig: o}\includegraphics[trim={0 0 0 0},width=0.235\textwidth]{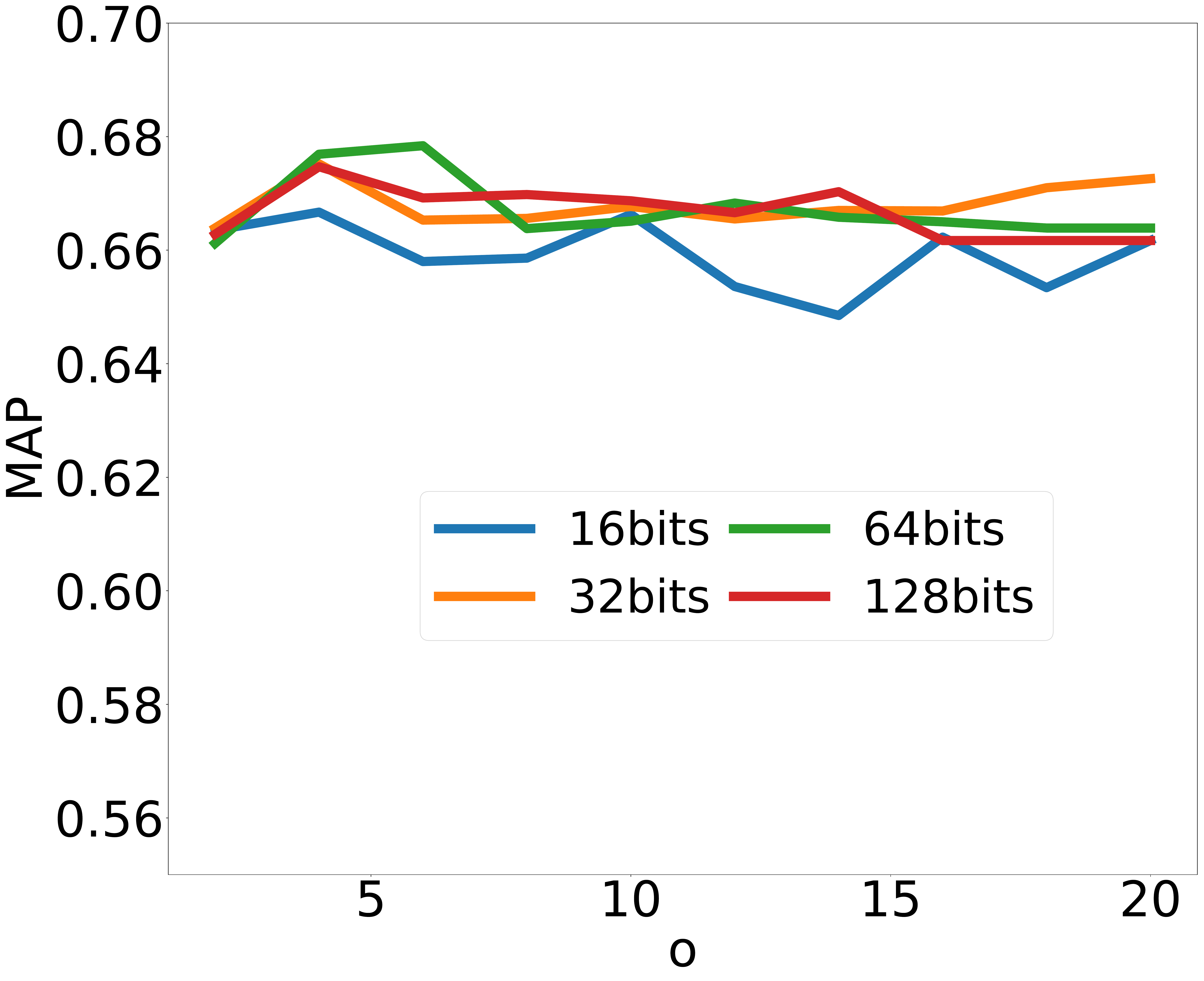}}
\subfigure[MAP w.r.t. different $p$.]{\label{fig: p}\includegraphics[trim={0 0 0 0},width=0.235\textwidth]{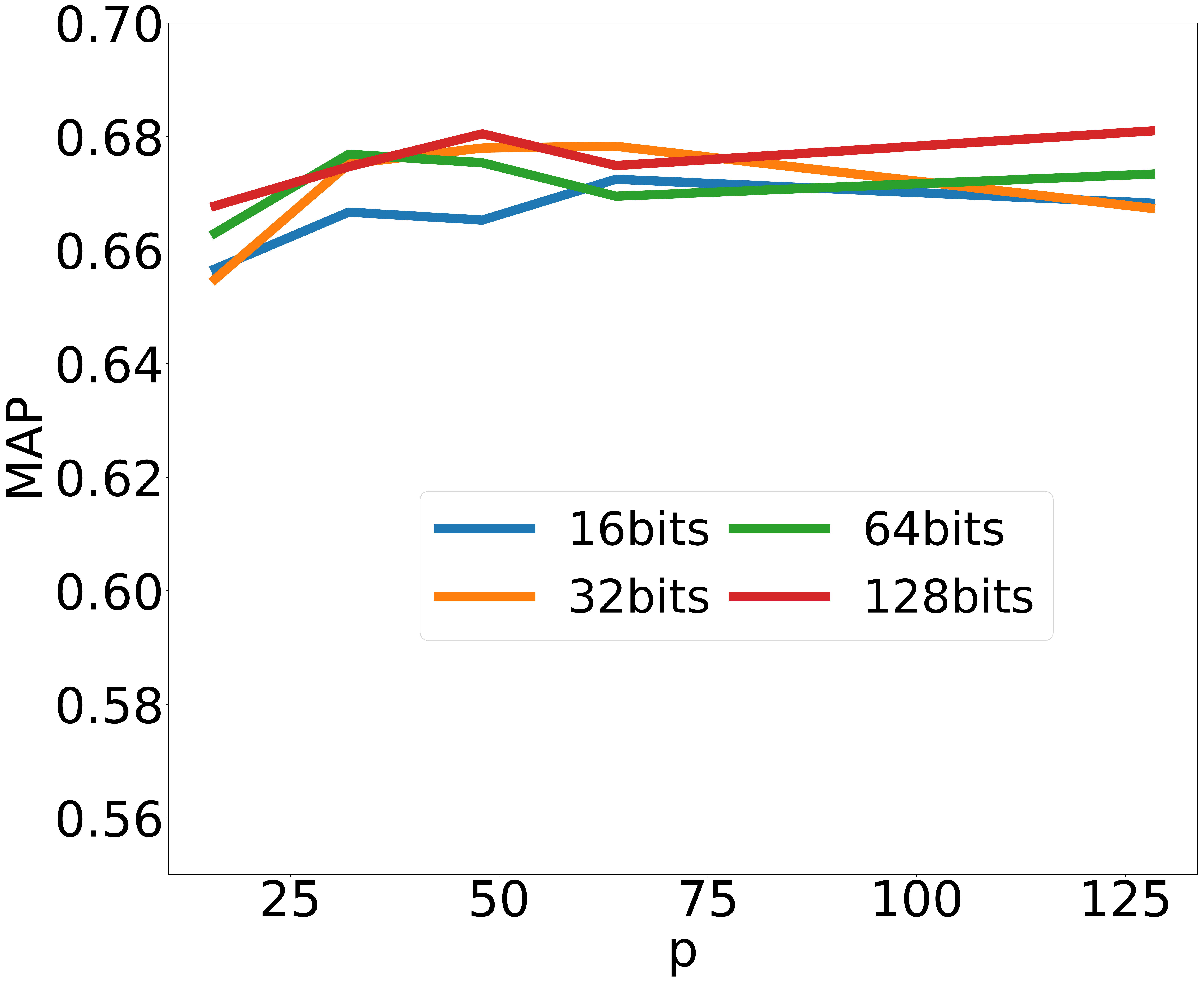}}

\caption{Parameter sensitive analysis for $o$ and $p$ on NUSWIDE.}
\label{fig: 4}
\vspace{-12pt}
\end{figure}
\section{Conclusions}
\label{conclusion}
This work presented a new unsupervised deep hashing approach for image search, namely DstilHash. Firstly, we theoretically investigated the relationship between the Bayes optimal classifier and noisy labels learned from local structures, showing that distilled data pairs can be potentially
collected. Secondly, with the above understanding, we provided a simple yet effective scheme to automatically distill data pairs. Thirdly, leveraging a distilled data set, we designed a deep hashing model and adopted a Bayesian learning framework to perform the hash code learning. The experimental results on three benchmark datasets demonstrated that the proposed DistillHash surpasses other
state-of-the-art methods.
\section{Acknowledgements}
{ This work was also supported in part by the National Natural Science Foundation of China under Grant 61572388 and 61703327, in part by the Key R\&D Program-The Key Industry Innovation Chain of Shaanxi under Grant 2017ZDCXL-GY-05-04-02, 2017ZDCXL-GY-05-02 and 2018ZDXM-GY-176, in part by the National Key R\&D Program of China under Grant 2017YFE0104100, and in part by the Australian Research Council Projects DP-180103424, DE-1901014738, and FL-170100117.}
{\small
\bibliographystyle{ieee}
\bibliography{egpaper_for_review}
}

\end{document}